\documentclass[11pt]{article}

\usepackage{fullpage}
\usepackage[round]{natbib}
\usepackage{algorithm}
\usepackage[noend]{algorithmic}
\usepackage{amsmath,amsthm,amsfonts,amssymb}
\usepackage{amsmath}
\usepackage{hyperref}
\usepackage{color}
\usepackage{mathrsfs}
\usepackage{enumitem}
\usepackage{bm}
\usepackage{multirow}
\usepackage{booktabs}
\usepackage{makecell}
\usepackage{graphicx}
\usepackage{subfigure}
\usepackage{caption}
\usepackage{thmtools}
\usepackage{thm-restate}
\usepackage{setspace}
\usepackage{makecell}
\definecolor{light-gray}{gray}{0.85}
\usepackage{colortbl}
\usepackage{xcolor}
\usepackage{hhline}

\usepackage{natbib}
\usepackage{algorithm}
\usepackage[noend]{algorithmic}
\usepackage{amsthm}
\usepackage{amsmath,amsfonts,amssymb}
\usepackage{hyperref}
\usepackage{color}
\usepackage{mathrsfs}
\usepackage{enumitem}
\usepackage{bm}
\usepackage{multirow}
\usepackage{booktabs}
\usepackage{makecell}
\usepackage{graphicx}
\usepackage{caption}
\usepackage{thmtools}
\usepackage{thm-restate}
\usepackage{setspace}
\usepackage{makecell}
\definecolor{light-gray}{gray}{0.85}
\usepackage{colortbl}
\usepackage{xcolor}
\usepackage{hhline}

\hypersetup{
    colorlinks,
    linkcolor={blue!50!black},
    citecolor={blue!50!black},
}
\colorlet{linkequation}{blue}

  {\list{}{\leftmargin=0.3in\rightmargin=0.3in}\item[]}%
  {\endlist}
\usepackage{amsfonts}

\newcommand{\vect}[1]{\ensuremath{\mathbf{#1}}}
\newcommand{\mat}[1]{\ensuremath{\mathbf{#1}}}

\newcommand{\argmin}{\mathop{\mathrm{argmin}}}
\newcommand{\argmax}{\mathop{\mathrm{argmax}}}

\renewcommand{\det}{\mathrm{det}}

\newcommand{\trans}{^{\top}}

\newcommand{\E}{\mathbb{E}}
\renewcommand{\Pr}{\mathbb{P}}
\newcommand{\Var}{\text{Var}}

\newcommand{\paren}[1]{{\left( #1 \right)}}

\newcommand{\one}{{\mathbf1}}

\newcommand{\Tcal}{\mathcal{T}}
\newcommand{\Xcal}{\mathcal{X}}

\newcommand{\Vcal}{\mathcal{V}}

\newcommand{\fhat}{\hat{f}}

\newcommand{\Qstar}{Q^{\star}}

\newcommand{\cC}{\mathcal{C}}

\newcommand{\tlO}{\mathcal{\tilde{O}}}

\newcommand{\N}{\mathbb{N}}
\newcommand{\R}{\mathbb{R}}

\newcommand{\I}{\mat{I}}

\newcommand{\V}{\mat{V}}

\newcommand{\x}{\vect{x}}

\newcommand{\z}{\vect{z}}

\newcommand{\cS}{\mathcal{S}}
\newcommand{\cA}{\mathcal{A}}

\newcommand{\cF}{\mathcal{F}}

\newcommand{\cN}{\mathcal{N}}

\newcommand{\cH}{\mathcal{H}}
\newcommand{\cT}{\mathcal{T}}

\newtheorem{theorem}{Theorem}
\newtheorem{lemma}[theorem]{Lemma}
\newtheorem{corollary}[theorem]{Corollary}

\newtheorem{assumption}[theorem]{Assumption}
\newtheorem{proposition}[theorem]{Proposition}
\newtheorem{remark}[theorem]{Remark}

\newtheorem{definition}[theorem]{Definition}

\newcommand{\Fcal}{\mathcal{F}}
\newcommand{\Dcal}{\mathcal{D}}
\newcommand{\Acal}{\mathcal{A}}
\newcommand{\Scal}{\mathcal{S}}
\newcommand{\Lcal}{\mathcal{L}}

\newcommand{\Gcal}{\mathcal{G}}
\newcommand{\Zcal}{\mathcal{Z}}
\newcommand{\Ocal}{\mathcal{O}}
\newcommand{\Ecal}{\mathcal{E}}
\newcommand{\EcalII}{\mathcal{E}_{\textrm{V}}}
\newcommand{\Bcal}{\mathcal{B}}
\newcommand{\Ncal}{\mathcal{N}}

\newcommand{\reg}{ {\mathrm{Reg}}}

\newcommand{\T}{{\mathbb{T}}}

\newcommand{\Mcal}{\mathcal{M}}

\newcommand{\pistar}{\pi^\star}

\newcommand{\olive}{\textsc{Olive}}
\newcommand{\golf}{\textsc{Golf}}

\newcommand{\algx}{\golf}

\newcommand{\eleanor}{\textsc{Eleanor}}

\newcommand{\eludim}{\dim_\mathrm{E}}
\newcommand{\dedim}{\dim_\mathrm{DE}}
\newcommand{\BEdim}{\dim_\mathrm{BE}}
\newcommand{\deber}{\textsf{DE}}
\newcommand{\be}{\textsf{BE}}

\newcommand{\bedim}{$\dim_{\rm{BE}}$}
\newcommand{\bedimII}{$\dim_{\rm{VBE}}$}

\newcommand{\dirac}{\Delta}

\newcommand{\Ffrak}{\mathfrak{F}}
\newcommand{\hatberrII}[1]{\hat{\Ecal}_{\textrm V}(#1)} 
\newcommand{\tildeberrII}[1]{\tilde{\Ecal}_{\textrm V}(#1)}

\newcommand{\nact}{n_\text{act}}
\newcommand{\nelim}{n_\text{elim}}
\newcommand{\zetaact}{\zeta_\text{act}}
\newcommand{\zetaelim}{\zeta_\text{elim}}

\newcommand{\berr}[1]{\Ecal(#1)}

\newcommand{\pit}[1]{\pi_{#1}}
\newcommand{\hatberr}[1]{\hat \Ecal(#1)} 

\newcommand{\dlin}{d_{\text{lin}}}
\newcommand{\dE}{d_{\text{E}}}

\usepackage{times}

\begin{document}

\title{Bellman Eluder Dimension: New Rich Classes of RL Problems, and Sample-Efficient Algorithms}

\author{%
 Chi Jin  \thanks{Princeton University. Email: \texttt{chij@princeton.edu}}
 \and
 Qinghua Liu
 \thanks{Princeton University. Email: \texttt{qinghual@princeton.edu}}
 \and
 Sobhan Miryoosefi  \thanks{Princeton University. Email: \texttt{miryoosefi@cs.princeton.edu}}
}

\date{February 1, 2021;\quad Revised: June 12, 2021 }

\maketitle

 \begin{abstract}


Finding the minimal structural assumptions that empower sample-efficient learning is one of the most important research directions in  Reinforcement Learning (RL). This paper advances our understanding of this fundamental question by introducing a new complexity measure---Bellman Eluder (BE) dimension. We show that the family of RL problems of low BE dimension is remarkably rich, which subsumes a vast majority of existing tractable RL problems including but not limited to tabular MDPs, linear MDPs, reactive POMDPs, low Bellman rank problems as well as low Eluder dimension problems. This paper further designs a new optimization-based algorithm---\golf, and reanalyzes a hypothesis elimination-based algorithm---{\olive} \citep[proposed in][]{jiang2017contextual}. We prove that both algorithms learn the near-optimal policies of low BE dimension problems in a number of samples that is polynomial in all relevant parameters, but independent of the size of state-action space. Our regret and sample complexity results match or improve the best existing results for several well-known subclasses of low BE dimension problems.
 \end{abstract}

\tableofcontents{}

\section{Introduction}
\label{sec:intro}

Modern Reinforcement Learning (RL) commonly engages practical problems with an enormous number of states, where \emph{function approximation}
must be deployed to approximate the true value function using functions from a prespecified function class. Function approximation, especially based on deep neural networks, lies at the heart of the recent practical successes of RL in domains such as Atari \citep{mnih2013playing}, Go \citep{silver2016mastering}, robotics \citep{kober2013reinforcement}, and dialogue systems \citep{li2016deep}.

Despite its empirical success, RL with function approximation raises a new series of theoretical challenges when comparing to the classic tabular RL: (1) \emph{generalization}, to generalize knowledge from the visited states to the unvisited states due to the enormous state space. (2) \emph{limited expressiveness}, to handle the complicated issues where true value functions or intermediate steps computed in the algorithm can be functions outside the prespecified function class. (3) \emph{exploration}, to address the tradeoff between exploration and exploitation when above challenges are present. 

Consequently, most existing theoretical results on efficient RL with function approximation rely on relatively strong structural assumptions. For instance, many require that the MDP admits a linear approximation \citep{wang2019optimism,jin2020provably,zanette2020learning}, or that the model is precisely Linear Quadratic Regulator (LQR) \citep{anderson2007optimal,fazel2018global,dean2019sample}. Most of these structural assumptions rarely hold in practical applications. This naturally leads to one of the most fundamental questions in RL.
\begin{center}
\textbf{What are the minimal structural assumptions that empower sample-efficient RL?}
\end{center}

We advance our understanding of this grand question via the following two steps: (1) identify a rich class of RL problems (with weak structural assumptions) that cover many practical applications of interests; (2) design sample-efficient algorithms that provably learn any RL problem in this class.


The attempts to find weak or minimal structural assumptions that allow statistical learning can be traced in supervised learning where VC dimension \citep{vapnik2013nature} or Rademacher complexity \citep{bartlett2002rademacher} is proposed, or in online learning where Littlestone dimension \citep{littlestone1988learning} or sequential Rademacher complexity \citep{rakhlin2010online} is developed. 

In the area of reinforcement learning, there are two intriguing lines of recent works that have made significant progress in this direction. 
To begin with,  \citet{jiang2017contextual} introduces a generic complexity notion---Bellman rank, which can be proved small for many RL problems including linear MDPs \citep{jin2020provably}, reactive POMDPs \citep{krishnamurthy2016pac}, etc. \cite{jiang2017contextual} further propose an hypothesis elimination-based algorithm---\olive\ for sample-efficient learning of problems with low Bellman rank. On the other hand,  recent work by \citet{wang2020provably} considers general function approximation with low Eluder dimension \citep{russo2013eluder}, and designs a UCB-style algorithm with regret guarantee. Noticeably, generalized linear MDPs \citep{wang2019optimism} and kernel MDPs (see Appendix \ref{app:examples}) are subclasses of low Eluder dimension problems, but not low Bellman rank.

 \begin{figure}
 \centering
  \includegraphics[width=0.6\textwidth]{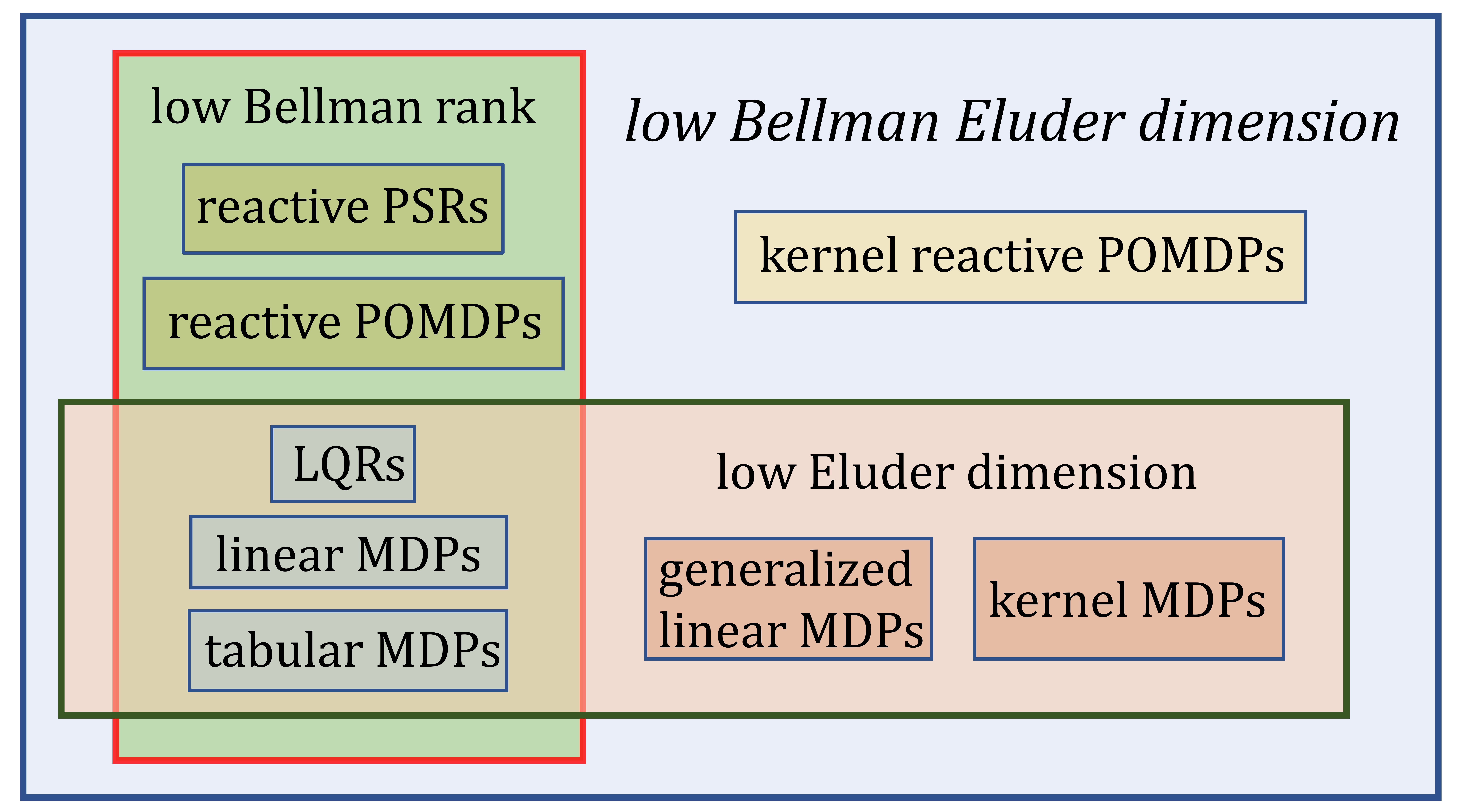}
\caption{A schematic summarizing relations among families of RL problems\protect\footnotemark}
\label{fig:relations}
 \end{figure}
 \footnotetext{The family of low Bellman rank problems and low Bellman Eluder dimension problems include both Q-type and V-type variants. Please refer to Section \ref{subsec:relations} and Appendix \ref{app:BE-typeII} for more details.}

In this paper, we make the following three contributions.
\begin{itemize}
	\item We introduce a new complexity measure for RL---Bellman Eluder (BE) dimension. We prove that the family of RL problems of low BE dimension is remarkably rich, which subsumes both low Bellman rank problems and low Eluder dimension problems---two arguably most generic tractable function classes so far in the literature  (see Figure \ref{fig:relations}). The family of low BE dimension further includes new problems such as kernel reactive POMDPs (see Appendix \ref{app:examples}) which were not known to be sample-efficiently learnable.

	\item We design a new optimization-based algorithm---\golf, which provably learns near-optimal policies of low BE dimension problems in a number of samples that is polynomial in all relevant parameters, but independent of the size of state-action space. Our regret or sample complexity guarantees match \cite{zanette2020learning} which is minimax optimal when specified to the linear setting. Our rates further improve upon \cite{jiang2017contextual,wang2020provably} in low Bellman rank  and low Eluder dimension settings,  respectively.
	\item We reanalyze the hypothesis elimination based algorithm---{\olive} proposed in \cite{jiang2017contextual}. We show it can also learn RL problems with low BE dimension sample-efficiently, under slightly weaker assumptions but with worse sample complexity comparing to {\golf}.
\end{itemize}

\subsection{Related works}
\label{sec:relat}
This section reviews prior theoretical works on RL, under Markov Decision Process (MDP) models.

We remark that there has been a long line of research on function approximation in the \emph{batch RL} setting \citep[see, e.g.,][]{szepesvari2005finite,munos2008finite,chen2019information,xie2020batch}.
In this setting, agents are provided with exploratory data or simulator, so that they do not need to explicitly address the challenge of exploration. 
In this paper, we do not make such assumption, and attack the exploration problem directly. In the following we focus exclusively on the RL results in the general setting where exploration is required.

\paragraph{Tabular RL.}
Tabular RL concerns MDPs with a small number of states and actions, which has been thoroughly studied in recent years
\citep[see, e.g.,][]{brafman2002r,jaksch2010near,dann2015sample,agrawal2017optimistic,azar2017minimax,zanette2019tighter,jin2018q,zhang2020almost}. 
In the episodic setting with non-stationary dynamics, 
the best regret bound $\tilde{\Ocal}(\sqrt{H^2|\Scal||\Acal|T})$ is achieved by both model-based \citep{azar2017minimax} and model-free \citep{zhang2020almost} algorithms. 
Moreover, the bound is proved to be minimax-optimal \citep{jin2018q,domingues2021episodic}. 
This minimax bound suggests that when the state-action space is enormous, RL is information-theoretically hard without further structural assumptions.


\paragraph{RL with linear function approximation.}
A recent line of work 
 studies RL with linear function approximation \citep[see, e.g.,][]{jin2020provably,wang2019optimism,cai2019provably,zanette2020learning,zanette2020provably,agarwal2020flambe,neu2020unifying,sun2019model}
 These papers assume certain completeness conditions, as well as the optimal value function can be well approximated by linear functions. Under one formulation of linear approximation, the minimax regret bound $\tilde{\Ocal}(d\sqrt{T})$ is achieved by algorithm \eleanor\  \citep{zanette2020learning}, where $d$ is the ambient dimension of the feature space.

 
\vspace{-0.6ex}

\paragraph{RL with general function approximation.}
Beyond the linear setting, there is a flurry line of research studying RL with general function approximation
\citep[see, e.g.,][]{osband2014model,jiang2017contextual,sun2019model,dong2020root,wang2020provably,yang2020bridging,foster2020instance}. Among them, \cite{jiang2017contextual} and \cite{wang2020provably} are the closest to our work.

\citet{jiang2017contextual} propose a complexity measure named Bellman rank and design an algorithm \olive\ with  PAC guarantees for problems with low Bellman rank. We note that low Bellman rank is a special case of low BE dimension. When specialized to the low Bellman rank setting, our result  for {\olive} exactly matches the guarantee in \cite{jiang2017contextual}. Our result for {\golf} requires an additional completeness assumption, but provides sharper sample complexity guarantee.

\citet{wang2020provably} propose a UCB-type algorithm with a regret guarantee under the assumption that the function class has a low eluder dimension. Again, we will show that low Eluder dimension is a special case of low BE dimension. Comparing to \cite{wang2020provably}, our algorithm {\golf} works under a weaker completeness assumption, with a better regret guarantee.

\vspace{-0.6ex}

\paragraph{Relation to bilinear classes} Concurrent to this work, \citet{du2021bilinear} propose a new general tractable class of RL problems---bilinear class with low effective dimension (also known as low critical information gain in \citet{du2021bilinear}). We comment on the similarities and differences between two works as follows.

In terms of algorithms, both Algorithm \ref{alg:olive} in this paper and the algorithm proposed in \citet{du2021bilinear} are based on \olive~originally proposed in \citet{jiang2017contextual}. The two algorithms share similar guarantees in terms of assumptions and complexity results. More importantly, our work further develops a new type of algorithm for general function approximation---\golf, a natural and clean algorithm which can be viewed as an optimistic version of classical algorithm---Fitted Q-Iteration \citep{szepesvari2010algorithms}. \golf~gives much sharper sample complexity guarantees compared to \cite{du2021bilinear} ~for various settings, and is minimax-optimal when applied to the linear setting \citep{zanette2020learning}.

In terms of richness of new classes identified, it depends on (a) what structure of MDP the complexity measures are applied to, and (b) what complexity measures are used. For (a), BE dimension applies to the Bellman error, while the bilinear class allows  general surrogate losses of the Bellman error. For (b), this paper uses Eluder dimension while \citet{du2021bilinear} uses effective dimension. 
It can be shown that low effective dimension always implies low Eluder dimension (see Appendix \ref{subsec:eff-dim}). 
In short, \citet{du2021bilinear} is more general in (a), while our work is more general in (b). As a result, neither work fully captures the other.

In particular, our BE framework covers a majority of the examples identified in \citet{du2021bilinear} including
low occupancy complexity, linear $Q^\star/V^\star$, $Q^\star$ state aggregation, feature selection/FLAMBE. Nevertheless, our work can not address examples with model-based function approximation (e.g., low witness rank \cite{sun2019model}) while \cite{du2021bilinear} can. On the other hand, \citet{du2021bilinear} can not address the class of RL problems with low Eluder dimension \citep{wang2020provably} while our work can. Moreover, for several classes of RL problems that both works cover, our complexity measure is sharper. For example, in the setting of function approximation with generalized linear functions, the BE dimension is $\tilde{O}(d)$ where $d$ is the ambient dimension of the feature vectors, while the effective dimension under the generalized bilinear framework of \citet{du2021bilinear} is at least $\tilde{\Omega}(d^2)$.

\section{Preliminaries}
\label{sec:prelim}

We consider episodic Markov Decision Process (MDP), denoted by 
	$\Mcal = (\Scal,\Acal,H,\Pr,r)$, where $\Scal$ is the state space, $\Acal$ is the action space, $H$ is the number of steps in each episode, $\Pr=\{\Pr_h\}_{h\in[H]}$ is the collection of transition measures with $\Pr_h(s'\mid s,a)$ equal to the probability of transiting to $s'$ after taking action $a$ at state $s$ at the $h^{\text{th}}$ step, and $r=\{r_h\}_{h\in[H]}$ is the collection of reward functions with $r_h(s,a)$ equal to the deterministic reward received  after taking action $a$ at state $s$ at the $h^{\text{th}}$ step.
	\footnote{We study deterministic reward  for notational simplicity. Our results readily generalize to random rewards.} 
Throughout this paper, we assume reward is non-negative, and 
$\sum_{h=1}^H r_h(s_h, a_h) \leq 1$  for all possible sequence $(s_1, a_1,\dots, s_H, a_H)$.


In each episode, the agent starts at a \emph{fixed} initial state $s_1$. 
Then, at each step $h\in[H]$, the agent observes its current state $s_h$, takes action $a_h$, receives reward $r_h(s_h,a_h)$, and causes the environment to transit to $s_{h+1}\sim \Pr_h(\cdot\mid s_h,a_h)$. 
Without loss of generality, we assume there is a terminating state $s_{\text{end}}$ which the environment will \emph{always} transit to  at step $H+1$, and the episode terminates when $s_{\text{end}}$ is reached.


\paragraph{Policy and value functions}
A (deterministic) policy $\pi$  is a collection of $H$ functions $\{\pi_h:\ \Scal\rightarrow \Acal\}_{h=1}^{H}$.
We denote $V^\pi_h : \Scal \rightarrow \mathbb{R}$ as the value function at step $h$ for policy $\pi$, so that $V^\pi_h(s)$
gives the expected sum of the remaining rewards received under policy $\pi$, starting from $s_h = s$, till the end of the episode. In symbol,
\begin{equation*}
	V_h^\pi(s):= \E_{\pi}[\sum_{h'=h}^H r_{h'}(s_{h'},a_{h'}) \mid s_h = s ].
\end{equation*}
Similarly, we denote $Q^\pi_h : \Scal \times \Acal \rightarrow \mathbb{R}$ as the $Q$-value function at step $h$ for policy $\pi$, where
\begin{equation*}
	Q_h^\pi(s,a):= \E_{\pi}[\sum_{h'=h}^H r_{h'}(s_{h'},a_{h'}) \mid s_h = s, a_h=a ].
\end{equation*}


There exists an optimal policy $\pistar$, which gives the optimal value function for all states \citep{puterman2014markov}, in the sense, $V^{\pistar}_h(s) = \sup_{\pi} V^\pi_h(s)$ for all $h\in[H]$ and $s\in\Scal$. 
For notational simplicity, we abbreviate $V^{\pistar}$ as $V^\star$. We similarly define the optimal $Q$-value function as 
$Q^{\star}$. Recall that $Q^{\star}$ satisfies the Bellman optimality equation:
\begin{equation} \label{eq:bellman_op}
Q^\star_h(s, a) = (\Tcal_h Q^\star_{h+1})(s,a) := r_h(s,a) + \E_{s' \sim \Pr_h(\cdot\mid s,a)}\max_{a' \in \Acal} Q^\star_{h+1}(s',a').
\end{equation}
for all $(s, a, h) \in \cS \times \cA \times [H]$. We also call $\Tcal_h$ the \emph{Bellman operator} at step $h$.



\paragraph{$\epsilon$-optimality and regret} 
We say a policy $\pi$ is $\epsilon$-optimal if $V^\pi_1(s_1)\ge V^\star_1(s_1)-\epsilon$. 
Suppose an agent interacts with the environment for $K$ episodes. Denote by $\pi^k$ the policy  the agent follows in episode $k \in [K]$. The (accumulative) regret is defined as 
\begin{equation*}
\reg(K): = \sum_{k=1}^K [V_1^\star(s_1) - V_1^{\pi^k}(s_1)].
\end{equation*}
The objective of reinforcement learning is to find an $\epsilon$-optimal policy within a small number of interactions or to achieve sublinear regret.

\subsection{Function approximation}

In this paper, we consider reinforcement learning with value function approximation. 
Formally, the learner is given a  function class $\Fcal=\Fcal_1\times\cdots\times\Fcal_H$, where $\Fcal_h \subseteq (\Scal \times \Acal \rightarrow [0,1])$ offers a set of candidate functions to approximate $Q^\star_h$---the optimal $Q$-value function at step $h$. Since no reward is collected in the $(H+1)^{\text{th}}$ steps, we always set $f_{H+1} =0$.

Reinforcement learning with function approximation in general is extremely challenging without further assumptions (see, e.g., hardness results in \cite{krishnamurthy2016pac,weisz2020exponential}).
Below, we present two assumptions about function approximation that are commonly adopted in the literature.

\begin{assumption}[Realizability]
\label{asp:realizability}
$\Qstar_h \in \Fcal_h$ for all $h\in[H]$.
\end{assumption}

Realizability requires the function class is well-specified, i.e., function class $\cF$ in fact contains the optimal $Q$-value function $\Qstar$ with no approximation error.
\begin{assumption}[Completeness]
\label{asp:completeness}
	$\Tcal_h \Fcal_{h+1} \subseteq \Fcal_h$ for all $h\in[H]$. 
\end{assumption}

Note $\Tcal_h\Fcal_{h+1}$ is defined as $\{\Tcal_h f_{h+1}: f_{h+1}\in\Fcal_{h+1}\}$. Completeness requires the function class $\Fcal$ to be closed under the Bellman operator. 

When function class $\cF$ has finite elements, we can use its cardinality $|\cF|$ to measure the ``size'' of function class $\cF$. When  addressing function classes with infinite elements, we need a notion similar to cardinality. We use the standard $\epsilon$-covering number.

\begin{definition}[$\epsilon$-covering number]
	The $\epsilon$-covering number of a set $\mathcal{V}$ under metric $\rho$, denoted as $\Ncal(\mathcal{V}, \epsilon, \rho)$, is the minimum integer $n$ such that there exists a subset $\mathcal{V}_{o} \subset \mathcal{V}$ with $|\mathcal{V}_{o}| = n$, and for any $x \in \mathcal{V}$, there exists $y \in \mathcal{V}_{o}$ such that $\rho(x, y) \le \epsilon$.
\end{definition}

We refer readers to standard textbooks \citep[see, e.g.,][]{wainwright2019high} for further properties of covering number.
In this paper, we will always apply the covering number on function class $\cF = \Fcal_1\times\cdots\times\Fcal_{H}$, and use metric $\rho(f, g) = \max_h\|f_h-g_h\|_{\infty}$. For notational simplicity, we omit the metric dependence and denote the covering number as $\cN_\cF(\epsilon)$.

\subsection{Eluder dimension} 
\label{sec:eluder_dim}
One class of functions highly related to this paper is the function class of low Eluder dimension \citep{russo2013eluder}.

\begin{definition}[$\epsilon$-independence between points]
\label{def:ind_points}
	Let $\Gcal$ be a function class defined on $\Xcal$, and	$z$,$x_1,x_2$\\,$\ldots$,$x_n$$\in\Xcal$.
	We say 	$z$ is $\epsilon$-independent of $\{x_1,x_2,\ldots,x_n\}$ with respect to $\Gcal$ if there exist $g_1,g_2\in\Gcal$ such that $\sqrt{\sum_{i=1}^{n} ( g_1(x_i)-g_2(x_i))^2}\le \epsilon$, but $g_1(z)-g_2(z) > \epsilon$. 
\end{definition}


Intuitively, $z$ is independent of $\{x_1,x_2,\ldots,x_n\}$ means if that there exist two ``certifying'' functions $g_1$ and $g_2$, so that their function values are similar at all points $\{x_i\}_{i=1}^n$, but the values are rather different at $z$. 
This independence relation naturally induces the following complexity measure. 

\begin{definition}[Eluder dimension]
\label{def:eluder}
Let $\Gcal$ be a function class defined on $\Xcal$.
	The  Eluder dimension $\eludim(\Gcal,\epsilon)$ is the length of the longest sequence $\{x_1, \ldots, x_n\} \subset \Xcal$ such that there exists $\epsilon'\ge\epsilon$ where $x_i$ is $\epsilon'$-independent of $\{x_1, \ldots, x_{i-1}\}$ for all $i \in [n]$.
\end{definition}

Recall that a vector space has dimension $d$ if and only if $d$ is the length of the longest sequence of elements $\{x_1, \ldots, x_d\}$ such that $x_i$ is linearly independent of $\{x_1, \ldots, x_{i-1}\}$ for all $i \in [n]$. Eluder dimension generalizes the linear independence relation in standard vector space to capture both nonlinear independence and approximate independence, and thus is more general.

\section{Bellman Eluder Dimension}
\label{sec:DEdim}

In this section, we introduce our new complexity measure---Bellman Eluder (BE) dimension. As one of its most important properties, we will show that the family of problems with low BE dimension contains the two existing most general tractable problem classes in RL---problems with low Bellman rank, and problems with low Eluder dimension (see Figure \ref{fig:relations}).


%

We start by developing a new distributional version of the original Eluder dimension proposed by \citet{russo2013eluder}
(see Section \ref{sec:eluder_dim} for more details).
\begin{definition}[$\epsilon$-independence between distributions]
\label{def:ind_dist}
	Let $\Gcal$ be a function class defined on $\Xcal$, and	$\nu,\mu_1,\ldots,\mu_n$ be probability measures over $\Xcal$.
	We say 	$\nu$ is $\epsilon$-independent of $\{\mu_1,\mu_2,\ldots,\mu_n\}$ with respect to $\Gcal$ if there exists $g\in\Gcal$ such that  
	$\sqrt{\sum_{i=1}^{n} ( \E_{\mu_i} [g])^2}\le \epsilon$, but $|\E_{\nu}[g]| > \epsilon$. 
\end{definition}


\begin{definition}[Distributional Eluder (DE) dimension]
\label{def:DE}
Let $\Gcal$ be a function class defined on $\Xcal$, and $\Pi$ be a family of probability measures over $\Xcal$. 
	The  distributional Eluder dimension $\dedim(\Gcal,\Pi,\epsilon)$ is the length of the longest sequence $\{\rho_1, \ldots, \rho_n\} \subset \Pi$ such that there exists $\epsilon'\ge\epsilon$ where $\rho_i$ is $\epsilon'$-independent of $\{\rho_1, \ldots, \rho_{i-1}\}$ for all $i\in[n]$.
\end{definition}


Definition \ref{def:ind_dist} and Definition \ref{def:DE} generalize Definition \ref{def:ind_points} and Definition \ref{def:eluder} to their distributional versions, by inspecting the expected values of functions instead of the function values at points, and by restricting the candidate distributions to a certain family $\Pi$. The main advantage of this generalization is exactly in the statistical setting, where estimating the expected values of  functions with respect to a certain distribution family can be easier than estimating function values at each point (which is the case for RL in large state spaces).

It is clear that the standard Eluder dimension is a special case of the distributional Eluder dimension, because if we  choose $\Pi = \{\delta_x(\cdot) ~|~ x\in\Xcal\}$ where $\delta_x(\cdot)$ is the dirac measure centered at $x$, 
then $\eludim(\Gcal, \epsilon) = \dedim(\Gcal-\Gcal,\Pi,\epsilon)$ where $\Gcal-\Gcal=\{g_1-g_2:\ g_1,g_2\in\Gcal\}$.

Now we are ready to introduce the key notion in this paper---Bellman Eluder dimension.
\begin{definition}[Bellman Eluder (BE) dimension]
\label{defn:bedim}
Let $(I-\Tcal_h)\Fcal:=\{f_h-\Tcal_h f_{h+1}: \ f\in\Fcal\}$ be the set of Bellman residuals induced by $\Fcal$ at step $h$, and 
$\Pi=\{\Pi_h\}_{h=1}^{H}$ be a collection of $H$ probability measure families over $\Scal\times\Acal$. The $\epsilon$-Bellman Eluder of $\Fcal$ with respect to $\Pi$ is  defined as 
\begin{equation*}
\BEdim(\Fcal,\Pi,\epsilon) := 
	\max_{h\in[H]} \dedim\big((I-\Tcal_h)\Fcal,\Pi_h,\epsilon\big).
\end{equation*}
\end{definition}

\begin{remark}[Q-type v.s. V-type] Definition \ref{defn:bedim} is based on the Bellman residuals functions that take a state-action pair as input, thus referred to as Q-type BE dimension. Alternatively, one can define V-type BE dimension using a different set of Bellman residual functions that depend on states only (see Appendix \ref{app:BE-typeII}). 
We focus on Q-type in the main paper, and present the results for V-type in  Appendix \ref{app:BE-typeII}. Both variants are important, and they include different sets of examples (see Appendix \ref{app:BE-typeII}, \ref{app:examples}).
\end{remark}

In short, Bellman Eluder dimension is simply the distributional Eluder dimension on the function class of Bellman residuals, maximizing over all steps. In addition to function class $\cF$ and  error $\epsilon$,  Bellman Eluder dimension also depends on the choice of distribution family $\Pi$. For the purpose of this paper, we focus on the following two specific choices.
\begin{enumerate}
\item $\Dcal_{\Fcal}:=\{\Dcal_{\Fcal,h}\}_{h\in[H]}$, where $\Dcal_{\Fcal,h}$ denotes the collection of all probability measures over 
$\Scal\times\Acal$ at the $h^\text{th}$ step, which can be generated by executing the greedy policy $\pi_f$ induced by any $f\in\Fcal$, 
i.e., $\pi_{f, h} (\cdot) = \argmax_{a \in \cA} f_h (\cdot, a)$ for all $h \in [H]$.
\item $\Dcal_{\dirac}:=\{\Dcal_{\dirac,h}\}_{h\in[H]}$, where $\Dcal_{\dirac,h} = \{\delta_{(s, a)}(\cdot) | s\in\cS, a \in \cA\}$, i.e., the collections of probability measures that put measure $1$ on a single state-action pair.
\end{enumerate}

We say a RL problem has low BE dimension if $\min_{\Pi\in \{\Dcal_{\Fcal}, \Dcal_{\dirac}\}}\BEdim(\Fcal,\Pi,\epsilon)$ is small.

\subsection{Relations with known tractable classes of  RL problems}
\label{subsec:relations}


Known tractable problem classes in RL include but not limited to tabular MDPs, linear MDPs \citep{jin2020provably}, linear quadratic regulators \citep{anderson2007optimal}, generalized linear MDPs \citep{wang2019optimism}, kernel MDPs (Appendix \ref{app:examples}), reactive POMDPs \citep{krishnamurthy2016pac}, reactive PSRs \citep{singh2012predictive,jiang2017contextual}. There are two existing generic tractable problem classes that jointly contain all the examples mentioned above: the set of RL problems with low Bellman rank, and the set of RL problems with low Eluder dimension. However, for these two generic sets, one does not contain the other. 

In this section, we will show that our new class of RL problems with low BE dimension in fact contains both low Bellman rank problems and low Eluder dimension problems (see Figure \ref{fig:relations}). That is, our new problem class covers almost all existing tractable RL problems, and to our best knowledge, is the most generic tractable function class so far.






\paragraph{Relation with low Bellman rank} 
The seminal paper by \citet{jiang2017contextual} proposes the complexity measure---Bellman rank, and shows that a majority of RL examples mentioned above have low Bellman rank. They also propose a hypothesis elimination based algorithm---OLIVE, that learns any low Bellman rank problem within polynomial samples. Formally, 


\begin{definition}[Bellman rank]\label{def:bellman-rank-typeI}
The Bellman rank is the minimum integer $d$ so that there exists $\phi_h: \Fcal\rightarrow\R^d$ and $\psi_h:\Fcal\rightarrow\R^d$ for each $h\in[H]$, such that for any $f,f'\in\Fcal$, the average Bellman error.
\begin{equation*}
\Ecal(f,\pi_{f'},h):=\E_{\pi_{f'}} [ (f_h-\Tcal_h f_{h+1})(s_h,a_h)]	=\langle \phi_h(f),\psi_h(f')\rangle,
\end{equation*}
where $\|\phi_h(f)\|_2\cdot\|\psi_h(f')\|_2\le\zeta$, and $\zeta$ is the normalization parameter.
\end{definition}

We remark that similar to Bellman Eluder dimension, Bellman rank also has two variants---Q-type (Definition \ref{def:bellman-rank-typeI}) and V-type (see Appendix \ref{app:BE-typeII}).
Recall that we use $\pi_{f}$ to denote the greedy policy induced by value function $f$.
Intuitively, a problem with Bellman rank says its average Bellman error can be decomposed as the inner product of two $d$-dimensional vectors, where one vector depends on the roll-in policy $\pi_{f'}$, while the other vector depends on the value function $f$. At a high level, it claims that the average Bellman error has a linear inner product structure.


\begin{proposition}[low Bellman rank $\subset$ low BE dimension]
\label{prop:bellman-bedim}
If an MDP with function class $\cF$ has Bellman rank $d$ with normalization parameter $\zeta$, then
\begin{equation*}
\BEdim(\Fcal,\Dcal_{\Fcal},\epsilon)\le \Ocal(1+d\log(1+\zeta /\epsilon)).
\end{equation*}
\end{proposition}
Proposition \ref{prop:bellman-bedim} claims that problems with low Bellman rank also have low BE dimension, with a small multiplicative factor that is only logarithmic in $\zeta$ and $\epsilon^{-1}$.

\paragraph{Relation with low Eluder dimension}
\citet{wang2020provably} study the setting where the function class $\Fcal$ has low Eluder dimension, which includes generalized linear functions. They prove that, when the completeness assumption is satisfied,\footnote{\cite{wang2020provably} assume for any function $g$ (not necessarily in $\Fcal$), $\Tcal g \in \Fcal$, which is stronger than the completeness assumption presented in this paper (Assumption \ref{asp:completeness}).}
low Eluder dimension problems can be efficiently learned in polynomial samples.

\begin{proposition}[low Eluder dimension $\subset$ low BE dimension]
\label{prop:eluder-bedim}
	Assume $\Fcal$ satisfies completeness (Assumption \ref{asp:completeness}). Then for all $\epsilon>0$,
	\begin{equation*}
\BEdim\big(\Fcal,\Dcal_{\dirac},\epsilon\big)\le \max_{h\in[H]}\eludim(\Fcal_h,\epsilon).
\end{equation*}
\end{proposition}
Proposition \ref{prop:eluder-bedim} asserts that problems with low Eluder dimension also have low BE dimension, which is a natural consequence of completeness and the fact that Eluder dimension is a special case of distributional Eluder dimension. 

Finally, we show that the set of low BE dimension problems is strictly larger than the union of low Eluder dimension problems and low Bellman rank problems. 
\begin{proposition}[low BE dimension $\not\subset$ low Eluder dimension $\cup$ low Bellman rank]
\label{lem:lowerbound}
For any $m\in\N^{+}$, there exists an MDP and a function class $\Fcal$ so that
for all $\epsilon\in(0,1]$, we have $\BEdim(\Fcal,\Dcal_{\cF},\epsilon) = \BEdim(\Fcal,\Dcal_{\dirac},\epsilon)\le 5$,
but $\min\{\min_{h\in[H]}\dim_{\rm E}(\Fcal_h,\epsilon), {\rm Bellman\ rank}\}\ge m$.
\end{proposition}

In particular, the family of low BE dimension includes new examples such as kernel reactive POMDPs (Appendix \ref{app:examples}), which can not be addressed by the framework of either Bellman rank or Eluder dimension.

\section{Algorithm \golf}
\label{sec:lcgo}

Section \ref{sec:DEdim} defines a new class of RL problems with low BE dimension, and shows that the new class is rich, containing almost all the existing known tractable RL problems so far. In this section, we propose a new simple optimization-based algorithm---\textbf{G}lobal \textbf{O}ptimism based on \textbf{L}ocal \textbf{F}itting (\golf). We prove that, 
low BE dimension problems are indeed tractable, i.e., {\golf} can find near-optimal policies for these problems within a polynomial number of samples. 




\begin{algorithm}[t]
\caption{\golf  $(\Fcal,\Gcal,K,\beta)$ --- \textbf{G}lobal \textbf{O}ptimism based on \textbf{L}ocal \textbf{F}itting}
\label{alg:algx}
 \begin{algorithmic}[1]
 \STATE \textbf{Initialize}:  $\Dcal_1,\dots,\Dcal_H\leftarrow \emptyset$, $\Bcal^0 \leftarrow \Fcal$.
 \FOR{\textbf{episode} $k$ from $1$ to $K$} 
 \STATE \textbf{Choose policy} $\pi^k = \pi_{f^k}$, where $f^k = \argmax_{f \in \Bcal^{k-1}}f(s_1,\pit{f}(s_1))$. 
\label{algx-line:greedy}
\STATE \textbf{Collect} a trajectory $(s_1,a_1,r_1,\ldots, s_H,a_H,r_H,s_{H+1})$ by following $\pi^k$. 
\label{algx-line:rollin}
\STATE \textbf{Augment} $\Dcal_h=\Dcal_h\cup\{(s_h,a_h,r_h,s_{h+1})\}$ for all $h\in[H]$.
\label{algx-line:updateD}
\STATE \textbf{Update}
\vspace{-4mm}
\begin{equation*}
	\Bcal^{k}=\left\{ f \in \Fcal:\  \Lcal_{\Dcal_h}(f_h,f_{h+1}) \leq \inf_{g \in \Gcal_{h}} \Lcal_{\Dcal_h}(g,f_{h+1}) + \beta \ \mbox{for all }h\in[H]\right\},
\end{equation*}
\vspace{-3mm}
\hspace{+10mm}
\begin{equation}\label{eq:algx-loss}
\mbox{where }	\Lcal_{\Dcal_h}(\xi_{h},\zeta_{h+1}) = \sum_{(s,a,r,s') \in \Dcal_h}[\xi_h(s,a)-r -\max_{a' \in \Acal} \zeta_{h+1}(s',a')]^2.
\end{equation}
\label{algx-line:updateB}
\vspace{-2mm}
\ENDFOR
\STATE \textbf{Output} $\pi^\text{out}$ sampled uniformly at random from $\{\pi^k\}_{k=1}^{K}$.
 \end{algorithmic}
\end{algorithm}

At a high level, \golf~can be viewed as an optimistic version of the classic algorithm---Fitted Q-Iteration (FQI) \citep{szepesvari2010algorithms}. \golf~generalizes the \eleanor~algorithm \citep{zanette2020learning} from the special linear setting to the general setting with arbitrary function classes.

The pseudocode of {\golf} is given in Algorithm \ref{alg:algx}. {\golf} initializes  datasets $\{\Dcal_h\}_{h=1}^H$ to be empty sets, and confidence set $\Bcal^0$ to be $\cF$. Then, in each episode, {\golf} performs two main steps:
\begin{itemize}
	\item  Line \ref{algx-line:greedy} (Optimistic planning): compute the most optimistic value function $f^k$ from the confidence set $\Bcal^{k-1}$ constructed in the last episode
	, and choose $\pi^k$ to be its greedy policy. 
	\item Line \ref{algx-line:rollin}-\ref{algx-line:updateB} (Execute the policy and update the confidence set):  execute policy $\pi^k$ for one episode, collect data, and update the confidence set using the new data.
\end{itemize}

At the heart of \golf\ is the way we construct the confidence set $\Bcal^k$. For each $h\in[H]$, \golf\ maintains a \emph{local} regression constraint using the collected transition data $\Dcal_h$ at this step
\begin{equation} \label{eq:set_relaxed}
	\Lcal_{\Dcal_h}(f_h,f_{h+1}) \leq \inf_{g \in \Gcal_{h}} \Lcal_{\Dcal_h}(g,f_{h+1}) + \beta,
\end{equation}
where $\beta$ is a confidence parameter, and $\Lcal_{\Dcal_h}$ is the squared loss defined in \eqref{eq:algx-loss}, which can be viewed as a proxy to the squared Bellman error at step $h$. We remark that FQI algorithm \citep{szepesvari2010algorithms} simply updates $f_h \leftarrow \argmin_{\phi\in\Fcal_h}\Lcal_{\Dcal_h}(\phi,f_{h+1})$.
Our constraint \eqref{eq:set_relaxed} can be viewed as a relaxed version of this update, which allows $f_h$ to be not only the minimizer of the loss $\Lcal_{\Dcal_h}(\cdot,f_{h+1})$, but also any function whose loss is only slightly larger than the optimal loss over the auxiliary function class $\Gcal_{h}$.


We remark that in general, the optimization problem in Line \ref{algx-line:greedy} of {\golf} can not be solved computationally efficiently.

\subsection{Theoretical guarantees}
In this subsection, we present the theoretical guarantees for \golf, which hold under Assumption \ref{asp:realizability} (realizability) and the following generalized completeness assumption introduced in \cite{antos2008learning,chen2019information}. Let $\Gcal=\Gcal_1\times \dots \times \Gcal_H$ be an auxiliary function class provided to the learner  where each $\Gcal_h \subseteq (\Scal\times\Acal\rightarrow [0,1])$.
Generalized completeness requires the auxiliary function class $\Gcal$ to be rich enough so that applying Bellman operator to any function in the primary function class $\Fcal$ will end up in $\Gcal$.
\begin{assumption}[Generalized completeness]
\label{asp:G-completeness}
	$\Tcal_h \Fcal_{h+1} \subseteq \Gcal_h$ for all $h\in[H]$. 
\end{assumption}
If we choose $\Gcal=\Fcal$, then Assumption \ref{asp:G-completeness}  is equivalent to the standard completeness assumption (Assumption \ref{asp:completeness}). 
Now, we are ready to present the main theorem for \golf.
\begin{theorem}[Regret of \golf] \label{thm:main}
Under Assumption \ref{asp:realizability}, \ref{asp:G-completeness}, there exists an absolute constant $c$ such that for any $\delta\in(0,1]$, $K\in\N$, if we choose parameter $\beta = c \log[\Ncal_{\cF\cup\Gcal}(1/K)\cdot KH/\delta]$ in \golf, then with probability at least $1-\delta$, for all $k \in [K]$, we have $$\reg(k)=\sum_{t=1}^k \left[V_1^\star(s_1) - V_1^{\pi^t}(s_1)\right]\le \Ocal(H\sqrt{ dk\beta}),$$ where $d=\min_{\Pi\in\{\Dcal_{\dirac},\Dcal_{\Fcal}\}}
	 \BEdim\big(\Fcal,\Pi,1/\sqrt{K}\big)$ is the BE dimension.
\end{theorem}

Theorem \ref{thm:main} asserts that, under the realizability and completeness assumptions, the general class of RL problems with low BE dimension is indeed tractable: there exists an algorithm (\golf) that can achieve $\sqrt{K}$ regret, whose multiplicative factor depends only polynomially on the horizon of MDP $H$, the BE dimension $d$, and the log covering number of the two function classes. Most importantly, the regret is independent of the number of the states, which is crucial for dealing with practical RL problems 
with function approximation, where the 
state spaces are typically exponentially large. 

We remark that when function class $\Fcal\cup\Gcal$ has finite number of elements, its covering number is upper bounded by its cardinality $|\Fcal\cup\Gcal|$. For a wide range of function classes in practice, the log $\epsilon'$-covering number has only logarithmic dependence on $\epsilon'$. Informally, we denote the log covering number as $\log \Ncal_{\Fcal\cup\Gcal}$ and omit its $\epsilon'$ dependency for clean presentation. Theorem \ref{thm:main} claims that the regret scales as $\tilde{\Ocal}(H\sqrt{dK \log \Ncal_{\Fcal\cup\Gcal}})$, where $\tilde{\Ocal}(\cdot)$ omits absolute constants and logarithmic terms.\footnote{We will not omit $\log \Ncal_{\Fcal\cup\Gcal}$ in $\tilde{\Ocal}(\cdot)$ notation since for many function classes, $\log \Ncal_{\Fcal\cup\Gcal}$ is not small. For instance, for a  $\tilde{d}$-dimensional linear function class, $\log \Ncal_{\Fcal\cup\Gcal} = \tilde{\Ocal}(\tilde{d})$.}

By the standard online-to-batch argument, we also derive the sample complexity of \golf.

\begin{corollary}[Sample Complexity of \golf]
\label{cor:sample_golf}
Under Assumption \ref{asp:realizability}, \ref{asp:completeness}, there exists an absolute constant $c$ such that for any $\epsilon \in (0, 1]$, if we choose $\beta = c \log[\Ncal_{\cF\cup\Gcal}(\epsilon^2/(dH^2)) \cdot HK]$ in \golf, then the output policy $\pi^{\text{out}}$ is $\Ocal(\epsilon)$-optimal with probability at least $1/2$, if 
$$K \ge \Omega\left(\frac{H^2 d}{\epsilon^2}\cdot \log\left[\Ncal_{\Fcal\cup\Gcal}\left(\frac{\epsilon^2}{H^2d}\right) \cdot \frac{Hd}{\epsilon}\right]\right),$$ 
where $d=\min_{\Pi\in\{\Dcal_{\dirac},\Dcal_{\Fcal}\}} \BEdim\big(\Fcal,\Pi,\epsilon/H\big)$ is the BE dimension. 
\end{corollary}


Corollary \ref{cor:sample_golf} claims that $\tlO(H^2d\log(\Ncal_{\Fcal\cup\Gcal})/\epsilon^2)$ samples are enough for {\golf} to learn a near-optimal policy of any low BE dimension problem. Our sample complexity scales linear in both the BE dimension $d$, and the log covering number $\log(\Ncal_{\Fcal\cup\Gcal})$.

To showcase the sharpness of our results, we compare them to the previous results when restricted to the corresponding settings. (1) For linear function class with  ambient dimension $\dlin$, we have BE dimension $d = \tlO(\dlin)$ and $\log(\Ncal_{\cF\cup\Gcal}) = \tlO(\dlin)$. Our regret bound becomes $\tlO(H \dlin \sqrt{K})$ which matches the best known result \citep{zanette2020learning} up to logarithmic factors; (2) For function class with low Eluder dimension \citep{wang2020provably}, our results hold under weaker completeness assumptions. Our regret scales with $\sqrt{\dE}$ in terms of dependency on Eluder dimension $\dE$, which improves the linear $\dE$ scaling in the regret of \cite{wang2020provably}; (3) Finally, for low Bellman rank problems, our sample complexity scales linearly with Bellman rank, which improves upon the quadratic dependence in \cite{jiang2017contextual}. We remark that all results mentioned above assume (approximate) realizability. All except \cite{jiang2017contextual} assume (approximate) completeness.

\subsection{Key ideas in proving Theorem \ref{thm:main}} 
\label{subsec:algx-proof}

In this subsection, we present a brief proof sketch for the regret bound of {\golf}. We defer all the  details to Appendix \ref{appendix:lcgo}. For simplicity, we only discuss the case of choosing $\Dcal_{\cF}$ as the distribution family $\Pi$ in the definition of Bellman Eluder dimension (Definition \ref{defn:bedim}). The proof for using $\Dcal_{\dirac}$ as the distribution family follows from similar arguments.


Our proof strategy consists of  three main steps. 

\paragraph{Step 1: Prove optimism.} We firstly show that, with high probability, the optimal value function $Q^\star$ indeed lies in the confidence set $\Bcal^k$ for all $k \in [K]$ (Lemma \ref{lem:confiset-1} in Appendix \ref{appendix:golf-fullproof}), which is a natural consequence of martingale concentration and the properties of the confidence set we designed. Because of $Q^\star \in \Bcal^k$,  the optimistic planning step (Line \ref{algx-line:greedy}) in {\golf} guarantees that $V_1^\star(s_1) \le \max_a f_1^k(s_1, a) $ for every episode $k$. This optimism allows the following upper bound on regret
\begin{equation}\label{eq:linear_Bell_err}
\reg(K) \le \sum_{k=1}^K \paren{\max_{a}f^k_1(s_1,a) - V^{\pi^k}_1(s_1)} = \sum_{h=1}^{H} \sum_{k=1}^{K} \E_{\pi^k} \left[ (f_h^k -\cT f_{h+1}^k)(s_h,a_h)\right], 
\end{equation}
where the right equality follows from the standard policy loss decomposition (see, e.g., Lemma 1 in \cite{jiang2017contextual}), and $\E_{\pi}$ denotes the expectation taken over  sequence $(s_1, a_1, \ldots, s_H, a_H)$ when executing policy $\pi$.

\paragraph{Step 2: Utilize the sharpness of our confidence set.} Recall that our construction of the confidence set in Line \ref{algx-line:updateB} of {\golf} forces $f^k$ computed in episode $k$ to have a small loss $\Lcal_{\Dcal_h}$, which is a proxy for empirical squared Bellman error under data $\Dcal_h$. Since data  $\Dcal_h$ in episode $k$ are collected by executing each $\pi^i$ for one episode for all $i <k$, by standard martingale concentration arguments and the completeness assumption, we can show that with high probability (Lemma \ref{lem:confiset-2} in Appendix \ref{appendix:golf-fullproof})
\vspace{-2mm}
\begin{equation}
\label{eq:sq_Bell_err}
\sum_{i=1}^{k-1}  \E_{\pi^i} \left[ (f_h^k - \cT f_{h+1}^k)(s_h, a_h) \right]^2 {\le} \mathcal{O} ( \beta),\ \mbox{for all } (k,h)\in[K]\times[H].
\end{equation}

\paragraph{Step 3: Establish relations between \eqref{eq:linear_Bell_err} and \eqref{eq:sq_Bell_err}.}
So far, we want to upper-bound \eqref{eq:linear_Bell_err}, while we know \eqref{eq:sq_Bell_err}. We note that the RHS of \eqref{eq:linear_Bell_err} is very similar to the LHS of \eqref{eq:sq_Bell_err}, except that the latter is the squared Bellman error, and the expectation is taken under previous policy $\pi^i$ for $i <k$. To establish the connection between these two, it turns out that we need the Bellman Eluder dimension to be small. Concretely, we have the following lemma.
\begin{lemma} \label{lem:sketch}
Given a function class $\Phi$ defined on $\Xcal$ with $|\phi(x)|\le 1$ for all $(\phi,x)\in\Phi\times\Xcal$, and a family of probability measures $\Pi$ over $\Xcal$. 
	Suppose sequence $\{\phi_k\}_{k=1}^{K}\subset \Phi$ and $\{\mu_k\}_{k=1}^{K}\subset\Pi$ satisfy that for all $k\in[K]$,
	$\sum_{i=1}^{k-1} (\E_{\mu_i} [\phi_k])^2 \le \beta$. Then for all $k\in[K]$,
	$\sum_{i=1}^{k} |\E_{\mu_i} [\phi_i]| \le \Ocal(\sqrt{\dedim (\Phi,\Pi,1/k)\beta k}).$
\end{lemma}
Lemma \ref{lem:sketch} is a simplification of Lemma \ref{lem:de-regret} in Appendix \ref{appendix:lcgo}, which is a modification of Lemma 2 in \cite{russo2013eluder}.
Intuitively, Lemma \ref{lem:sketch} can be viewed as an analogue of the pigeon-hole principle for DE dimension. Choose $\Phi$ to be the function class of Bellman residuals, and $\mu_k$ to be the distribution under policy $\pi^k$, we finish the proof.

\section{Algorithm \olive} 
\label{sec:olive}

In this section, we analyze  algorithm  {\olive}\ proposed in \cite{jiang2017contextual}, which is based on hypothesis elimination. We prove that, despite {\olive} was originally designed for solving low Bellman rank problems, it naturally learns RL problems with low BE dimension as well.

The main advantage of \olive~comparing to \golf~is that \olive~does not require the completeness assumption. In return, \olive~has several disadvantages including worse sample complexity, and no sublinear regret.



\begin{algorithm}[t]
\caption{\textsc{Olive} $(\Fcal,\zeta_\text{act},\zeta_\text{elim},\nact,\nelim)$}
\label{alg:olive}
\begin{algorithmic}[1]
\STATE \textbf{Initialize}: $\Bcal^0 \gets \Fcal$, $\Dcal_h \leftarrow \emptyset$ for all $h,k$.
\FOR{\textbf{phase} $k=1,2,\ldots$} 
\STATE \label{olive-line:greedy}
\textbf{Choose policy} $\pi^k = \pi_{f^k}$, where $f^k = \argmax_{f \in \Bcal^{k-1}}f(s_1,\pit{f}(s_1))$. 
\STATE \label{olive-line:act-begin} \textbf{Execute} $\pi^k$ for $\nact$ episodes and \emph{refresh} $\Dcal_h$ to include the fresh $(s_h,a_h,r_h,s_{h+1})$ tuples.
\STATE \textbf{Estimate} $\hatberr{f^k, \pi^k, h}$ for all $h\in[H]$, where 
\vspace{-2mm}
\begin{align*}
\hatberr{g, \pi^k, h} = \frac{1}{|\Dcal_h|} \sum_{(s,a,r,s') \in \Dcal_h}\left(g_h(s,a)-r - \max_{a'\in\Acal}g_{h+1}(s',a')\right).
\end{align*}
\vspace{-3mm}
\IF{$\sum_{h=1}^H \hatberr{f^k, \pi^k,h} \le H\zeta_\text{act}$}\label{olive-line:if-act}
\STATE Terminate and output $\pi^k$. 
\label{olive-line:act-end}
\ENDIF
\STATE Pick any $t \in [H]$ for which $\hatberr{f^k, \pi^k,t} \ge \zeta_\text{act}$. \label{olive-line:elim-begin}
\STATE  \textbf{Execute} $\pi^k$ for $\nelim$ episodes and \emph{refresh} $\Dcal_h$ to include the fresh $(s_h,a_h,r_h,s_{h+1})$ tuples. 
\STATE \textbf{Estimate}  $\hatberr{f, \pi^k, t}$ for all $f\in\Fcal$.
\STATE \textbf{Update} 
$\Bcal^{k} = \left\{f \in \Bcal^{k-1} : \left|\hatberr{f,\pi^k,t} \right|\le \zeta_\text{elim} \right\}.$
\label{olive-line:elim-end}
\ENDFOR
\end{algorithmic}
\end{algorithm}

The pseudocode of {\olive} is presented in Algorithm \ref{alg:olive},  where in each phase the algorithm contains the following three main components:
\begin{itemize}
	\item Line \ref{olive-line:greedy} (Optimistic planning): compute the most optimistic value function $f^k$ from the candidate set $\Bcal^{k-1}$, and choose $\pi^k$ to be its greedy policy. 
	
	\item Line \ref{olive-line:act-begin}-\ref{olive-line:act-end} (Estimate Bellman error): estimate the Bellman error of $f^k$ under $\pi^k$; output $\pi^k$ if the estimated error is small, and otherwise activate the elimination procedure.
	
	\item  Line \ref{olive-line:elim-begin}-\ref{olive-line:elim-end} (Eliminate functions with large Bellman error): pick a step $t\in[H]$ where the  estimated Bellman error exceeds the activation threshold $\zeta_\text{act}$;
	eliminate all functions in the candidate set whose Bellman error at step $t$ exceeds the elimination threshold $\zeta_\text{elim}$.
\end{itemize}
We comment that \olive\ is computationally inefficient in general  because implementing the optimistic planning part requires solving an NP-hard problem in the worst case \citep[Theorem 4,][]{dann2018oracle}.


\subsection{Theoretical guarantees}
Now, we are ready to present the theoretical guarantee for \olive. 


\begin{theorem}[\olive]
\label{thm:olive}
Under Assumption \ref{asp:realizability}, there exists absolute constant $c$ such that if we choose 
$$\zeta_\text{act}=\frac{2\epsilon}{H},\ \zetaelim=\frac{\epsilon}{2H\sqrt{d}},\ \nact=\frac{H^2 \iota}{\epsilon^2},\text{ and } \nelim=\frac{H^2d\log(\Ncal_\Fcal(\zeta_\text{elim}/8)) \cdot \iota}{\epsilon^2}$$
 where $d=\BEdim(\Fcal,\Dcal_{\Fcal},\epsilon/H)$ and  $\iota=c\log(Hd/\delta\epsilon)$, 
 then with probability at least $1-\delta$, Algorithm \ref{alg:olive} will output an $\Ocal(\epsilon)$-optimal policy using at most  $\Ocal(H^3d^2\log[\Ncal_\Fcal(\zeta_\text{elim}/8)] \cdot \iota/{\epsilon^2})$ episodes.
\end{theorem}

Theorem \ref{thm:olive} claims that {\olive} learns an $\epsilon$-optimal policy of an MDP with BE dimension $d$ within $\tlO(H^3d^2\log(\Ncal_\Fcal) /\epsilon^2)$ episodes. When specialized to low Bellman rank problems, our sample complexity has the same quadratic dependence on Bellman rank $d$ as in \cite{jiang2017contextual}.

Comparing to {\golf}, the major advantage of {\olive} is that {\olive} does not require completeness assumption (Assumption \ref{asp:completeness}) to work. 
Nevertheless, {\olive} only learns the RL problems that have low BE dimension with respect to  distribution family $\Dcal_{\cF}$, not $\Dcal_{\dirac}$. The sample complexity of {\olive} is also worse than the sample complexity {\golf} (as presented in Corollary \ref{cor:sample_golf}).

Finally, we comment that interpreting {\olive} through the lens of BE dimension, makes the proof of Theorem \ref{thm:olive} surprisingly natural, which follows from the definition of BE dimension along with some standard concentration arguments. 

\subsection{Interpret \olive\ with BE dimension}
\label{subsec:olive-interpret}

In this subsection, we explain the key idea behind \olive\ through the lens of  BE dimension. 

To provide a clean high-level view, let us assume all estimates are accurate for now, and the activation threshold $\zetaact$ and the elimination threshold $\zetaelim$ satisfy $\zetaelim \sqrt{d}\le \zetaact$,  where $d= \text{\bedim}\big(\Fcal,\Dcal_{\Fcal},\zetaact\big)$.
Since $\berr{Q^\star,\pi,h}\equiv0$ for any $(\pi,h)$, $Q^\star$ is always in the candidate set. Therefore, the optimistic planning (Line \ref{olive-line:greedy}) guarantees $\max_{a}f^k_1(s_1,a)\ge V^\star_1(s_1)$. 

If the Bellman error summation is small (Line \ref{olive-line:if-act}) i.e., $\sum_{h=1}^H \berr{f^k, \pi^k,h} \le H\zeta_{\rm act}$, then by simple policy loss decomposition (e.g., Lemma 1 in \cite{jiang2017contextual}) and the optimism of $f^k$, $\pi^k$ is $ H\zeta_{\rm act}$-optimal. 
Otherwise, the elimination procedure is activated at some step $t$ satisfying $\berr{f^k, \pi^k,t} \ge \zeta_{\rm act}$ and all $f$ with $\berr{f, \pi^k,t} \ge \zeta_{\rm elim}$ get eliminated. 
The \emph{key} observation here is:
\begin{quote}
\sl	 If the elimination procedure is activated at step $h$ in phase $k_1<\ldots<k_m$, then the roll-in distribution of $\pi^{k_1},\ldots,\pi^{k_m}$ at step $h$ is an $\zetaact$-independent sequence with respect to the class of Bellman residuals $({I}-\Tcal_h)\Fcal$ at step $h$. Therefore, we should have $m \le d$.
\end{quote}
   For the sake of contradiction, assume $m \geq d+1$. 
   Let us prove $\pi^{k_1},\ldots,\pi^{k_{d+1}}$ is a $\zetaact$-independent sequence. 
   Firstly, for any $j\in[d+1]$, since $f^{k_j}$  is not eliminated in phase $k_1,\ldots,k_{j-1}$, we have 
    $$\sqrt{\sum_{i=1}^{j-1} \big(\Ecal(f^{k_j},\pi^{{k_{i}}},h)\big)^2} \le \sqrt{d}\times \zeta_{\rm elim} \le \zetaact.$$ 
    Besides, because the elimination procedure is activated at step $h$ in phase $k_j$, we have $\berr{f^{k_j},\pi^{k_j},h}\ge\zetaact$. 
    By Definition \ref{def:ind_dist}, we obtain that the roll-in distribution of $\pi^{k_j}$ at step $h$ is $\zetaact$-independent of those of $\pi^{k_1},\ldots,\pi^{k_{j-1}}$ for $j\in[d+1]$, which contradicts the definition $d=\text{\bedim}\big(\Fcal,\Dcal_{\Fcal},\zetaact\big)$.
As a result, the elimination procedure can happen at most $d$ times for each $h\in[H]$, which means the algorithm should terminate within $dH+1$ phases and output an $H\zeta_{\rm act}$-optimal policy.



\section{Conclusion}
\label{sec:conclu}

In this paper, we propose a new complexity measure---Bellman Eluder (BE) dimension for reinforcement learning with function approximation. 
Our new complexity measure identifies a new rich class of RL problems that subsumes a majority of existing tractable problem classes in RL.
We design a new optimization-based algorithm---\golf, and provide a new analysis for algorithm \olive. Both algorithms show that the new rich class of RL problems we identified in fact can   be learned within a polynomial number of samples. We hope our results shed light on the future research in finding the minimal structural assumptions that allow sample-efficient reinforcement learning.

\bibliographystyle{unsrtnat}
\bibliography{ref}

\begin{thebibliography}{49}
\providecommand{\natexlab}[1]{#1}
\providecommand{\url}[1]{\texttt{#1}}
\expandafter\ifx\csname urlstyle\endcsname\relax
  \providecommand{\doi}[1]{doi: #1}\else
  \providecommand{\doi}{doi: \begingroup \urlstyle{rm}\Url}\fi

\bibitem[Jiang et~al.(2017)Jiang, Krishnamurthy, Agarwal, Langford, and
  Schapire]{jiang2017contextual}
Nan Jiang, Akshay Krishnamurthy, Alekh Agarwal, John Langford, and Robert~E
  Schapire.
\newblock Contextual decision processes with low bellman rank are
  pac-learnable.
\newblock In \emph{International Conference on Machine Learning}, pages
  1704--1713. PMLR, 2017.

\bibitem[Mnih et~al.(2013)Mnih, Kavukcuoglu, Silver, Graves, Antonoglou,
  Wierstra, and Riedmiller]{mnih2013playing}
Volodymyr Mnih, Koray Kavukcuoglu, David Silver, Alex Graves, Ioannis
  Antonoglou, Daan Wierstra, and Martin Riedmiller.
\newblock Playing atari with deep reinforcement learning.
\newblock \emph{arXiv preprint arXiv:1312.5602}, 2013.

\bibitem[Silver et~al.(2016)Silver, Huang, Maddison, Guez, Sifre, Van
  Den~Driessche, Schrittwieser, Antonoglou, Panneershelvam, Lanctot,
  et~al.]{silver2016mastering}
David Silver, Aja Huang, Chris~J Maddison, Arthur Guez, Laurent Sifre, George
  Van Den~Driessche, Julian Schrittwieser, Ioannis Antonoglou, Veda
  Panneershelvam, Marc Lanctot, et~al.
\newblock Mastering the game of go with deep neural networks and tree search.
\newblock \emph{nature}, 529\penalty0 (7587):\penalty0 484--489, 2016.

\bibitem[Kober et~al.(2013)Kober, Bagnell, and Peters]{kober2013reinforcement}
Jens Kober, J~Andrew Bagnell, and Jan Peters.
\newblock Reinforcement learning in robotics: A survey.
\newblock \emph{The International Journal of Robotics Research}, 32\penalty0
  (11):\penalty0 1238--1274, 2013.

\bibitem[Li et~al.(2016)Li, Monroe, Ritter, Galley, Gao, and
  Jurafsky]{li2016deep}
Jiwei Li, Will Monroe, Alan Ritter, Michel Galley, Jianfeng Gao, and Dan
  Jurafsky.
\newblock Deep reinforcement learning for dialogue generation.
\newblock \emph{arXiv preprint arXiv:1606.01541}, 2016.

\bibitem[Wang et~al.(2019)Wang, Wang, Du, and Krishnamurthy]{wang2019optimism}
Yining Wang, Ruosong Wang, Simon~S Du, and Akshay Krishnamurthy.
\newblock Optimism in reinforcement learning with generalized linear function
  approximation.
\newblock \emph{arXiv preprint arXiv:1912.04136}, 2019.

\bibitem[Jin et~al.(2020)Jin, Yang, Wang, and Jordan]{jin2020provably}
Chi Jin, Zhuoran Yang, Zhaoran Wang, and Michael~I Jordan.
\newblock Provably efficient reinforcement learning with linear function
  approximation.
\newblock In \emph{Conference on Learning Theory}, pages 2137--2143, 2020.

\bibitem[Zanette et~al.(2020{\natexlab{a}})Zanette, Lazaric, Kochenderfer, and
  Brunskill]{zanette2020learning}
Andrea Zanette, Alessandro Lazaric, Mykel Kochenderfer, and Emma Brunskill.
\newblock Learning near optimal policies with low inherent bellman error.
\newblock \emph{arXiv preprint arXiv:2003.00153}, 2020{\natexlab{a}}.

\bibitem[Anderson and Moore(2007)]{anderson2007optimal}
Brian~DO Anderson and John~B Moore.
\newblock \emph{Optimal control: linear quadratic methods}.
\newblock Courier Corporation, 2007.

\bibitem[Fazel et~al.(2018)Fazel, Ge, Kakade, and Mesbahi]{fazel2018global}
Maryam Fazel, Rong Ge, Sham Kakade, and Mehran Mesbahi.
\newblock Global convergence of policy gradient methods for the linear
  quadratic regulator.
\newblock In \emph{International Conference on Machine Learning}, pages
  1467--1476. PMLR, 2018.

\bibitem[Dean et~al.(2019)Dean, Mania, Matni, Recht, and Tu]{dean2019sample}
Sarah Dean, Horia Mania, Nikolai Matni, Benjamin Recht, and Stephen Tu.
\newblock On the sample complexity of the linear quadratic regulator.
\newblock \emph{Foundations of Computational Mathematics}, pages 1--47, 2019.

\bibitem[Vapnik(2013)]{vapnik2013nature}
Vladimir Vapnik.
\newblock \emph{The nature of statistical learning theory}.
\newblock Springer science \& business media, 2013.

\bibitem[Bartlett and Mendelson(2002)]{bartlett2002rademacher}
Peter~L Bartlett and Shahar Mendelson.
\newblock Rademacher and gaussian complexities: Risk bounds and structural
  results.
\newblock \emph{Journal of Machine Learning Research}, 3\penalty0
  (Nov):\penalty0 463--482, 2002.

\bibitem[Littlestone(1988)]{littlestone1988learning}
Nick Littlestone.
\newblock Learning quickly when irrelevant attributes abound: A new
  linear-threshold algorithm.
\newblock \emph{Machine learning}, 2\penalty0 (4):\penalty0 285--318, 1988.

\bibitem[Rakhlin et~al.(2010)Rakhlin, Sridharan, and Tewari]{rakhlin2010online}
Alexander Rakhlin, Karthik Sridharan, and Ambuj Tewari.
\newblock Online learning: Random averages, combinatorial parameters, and
  learnability.
\newblock 2010.

\bibitem[Krishnamurthy et~al.(2016)Krishnamurthy, Agarwal, and
  Langford]{krishnamurthy2016pac}
Akshay Krishnamurthy, Alekh Agarwal, and John Langford.
\newblock Pac reinforcement learning with rich observations.
\newblock \emph{arXiv preprint arXiv:1602.02722}, 2016.

\bibitem[Wang et~al.(2020)Wang, Salakhutdinov, and Yang]{wang2020provably}
Ruosong Wang, Ruslan Salakhutdinov, and Lin~F Yang.
\newblock Provably efficient reinforcement learning with general value function
  approximation.
\newblock \emph{arXiv preprint arXiv:2005.10804}, 2020.

\bibitem[Russo and Van~Roy(2013)]{russo2013eluder}
Daniel Russo and Benjamin Van~Roy.
\newblock Eluder dimension and the sample complexity of optimistic exploration.
\newblock In \emph{Advances in Neural Information Processing Systems}, pages
  2256--2264, 2013.

\bibitem[Szepesv{\'a}ri and Munos(2005)]{szepesvari2005finite}
Csaba Szepesv{\'a}ri and R{\'e}mi Munos.
\newblock Finite time bounds for sampling based fitted value iteration.
\newblock In \emph{Proceedings of the 22nd international conference on Machine
  learning}, pages 880--887, 2005.

\bibitem[Munos and Szepesv{\'a}ri(2008)]{munos2008finite}
R{\'e}mi Munos and Csaba Szepesv{\'a}ri.
\newblock Finite-time bounds for fitted value iteration.
\newblock \emph{Journal of Machine Learning Research}, 9\penalty0
  (May):\penalty0 815--857, 2008.

\bibitem[Chen and Jiang(2019)]{chen2019information}
Jinglin Chen and Nan Jiang.
\newblock Information-theoretic considerations in batch reinforcement learning.
\newblock \emph{arXiv preprint arXiv:1905.00360}, 2019.

\bibitem[Xie and Jiang(2020)]{xie2020batch}
Tengyang Xie and Nan Jiang.
\newblock Batch value-function approximation with only realizability.
\newblock \emph{arXiv preprint arXiv:2008.04990}, 2020.

\bibitem[Brafman and Tennenholtz(2002)]{brafman2002r}
Ronen~I Brafman and Moshe Tennenholtz.
\newblock R-max-a general polynomial time algorithm for near-optimal
  reinforcement learning.
\newblock \emph{Journal of Machine Learning Research}, 3\penalty0
  (Oct):\penalty0 213--231, 2002.

\bibitem[Jaksch et~al.(2010)Jaksch, Ortner, and Auer]{jaksch2010near}
Thomas Jaksch, Ronald Ortner, and Peter Auer.
\newblock Near-optimal regret bounds for reinforcement learning.
\newblock \emph{Journal of Machine Learning Research}, 11\penalty0 (4), 2010.

\bibitem[Dann and Brunskill(2015)]{dann2015sample}
Christoph Dann and Emma Brunskill.
\newblock Sample complexity of episodic fixed-horizon reinforcement learning.
\newblock In \emph{Advances in Neural Information Processing Systems}, pages
  2818--2826, 2015.

\bibitem[Agrawal and Jia(2017)]{agrawal2017optimistic}
Shipra Agrawal and Randy Jia.
\newblock Optimistic posterior sampling for reinforcement learning: worst-case
  regret bounds.
\newblock In \emph{Advances in Neural Information Processing Systems}, pages
  1184--1194, 2017.

\bibitem[Azar et~al.(2017)Azar, Osband, and Munos]{azar2017minimax}
Mohammad~Gheshlaghi Azar, Ian Osband, and R{\'e}mi Munos.
\newblock Minimax regret bounds for reinforcement learning.
\newblock \emph{arXiv preprint arXiv:1703.05449}, 2017.

\bibitem[Zanette and Brunskill(2019)]{zanette2019tighter}
Andrea Zanette and Emma Brunskill.
\newblock Tighter problem-dependent regret bounds in reinforcement learning
  without domain knowledge using value function bounds.
\newblock \emph{arXiv preprint arXiv:1901.00210}, 2019.

\bibitem[Jin et~al.(2018)Jin, Allen-Zhu, Bubeck, and Jordan]{jin2018q}
Chi Jin, Zeyuan Allen-Zhu, Sebastien Bubeck, and Michael~I Jordan.
\newblock Is q-learning provably efficient?
\newblock In \emph{Advances in Neural Information Processing Systems}, pages
  4863--4873, 2018.

\bibitem[Zhang et~al.(2020)Zhang, Zhou, and Ji]{zhang2020almost}
Zihan Zhang, Yuan Zhou, and Xiangyang Ji.
\newblock Almost optimal model-free reinforcement learning via
  reference-advantage decomposition.
\newblock \emph{arXiv preprint arXiv:2004.10019}, 2020.

\bibitem[Domingues et~al.(2021)Domingues, M{\'e}nard, Kaufmann, and
  Valko]{domingues2021episodic}
Omar~Darwiche Domingues, Pierre M{\'e}nard, Emilie Kaufmann, and Michal Valko.
\newblock Episodic reinforcement learning in finite mdps: Minimax lower bounds
  revisited.
\newblock In \emph{Algorithmic Learning Theory}, pages 578--598. PMLR, 2021.

\bibitem[Cai et~al.(2019)Cai, Yang, Jin, and Wang]{cai2019provably}
Qi~Cai, Zhuoran Yang, Chi Jin, and Zhaoran Wang.
\newblock Provably efficient exploration in policy optimization.
\newblock \emph{arXiv preprint arXiv:1912.05830}, 2019.

\bibitem[Zanette et~al.(2020{\natexlab{b}})Zanette, Lazaric, Kochenderfer, and
  Brunskill]{zanette2020provably}
Andrea Zanette, Alessandro Lazaric, Mykel~J Kochenderfer, and Emma Brunskill.
\newblock Provably efficient reward-agnostic navigation with linear value
  iteration.
\newblock \emph{Advances in Neural Information Processing Systems}, 33,
  2020{\natexlab{b}}.

\bibitem[Agarwal et~al.(2020)Agarwal, Kakade, Krishnamurthy, and
  Sun]{agarwal2020flambe}
Alekh Agarwal, Sham Kakade, Akshay Krishnamurthy, and Wen Sun.
\newblock Flambe: Structural complexity and representation learning of low rank
  mdps.
\newblock \emph{Advances in Neural Information Processing Systems}, 33, 2020.

\bibitem[Neu and Pike-Burke(2020)]{neu2020unifying}
Gergely Neu and Ciara Pike-Burke.
\newblock A unifying view of optimism in episodic reinforcement learning.
\newblock \emph{Advances in Neural Information Processing Systems}, 33, 2020.

\bibitem[Sun et~al.(2019)Sun, Jiang, Krishnamurthy, Agarwal, and
  Langford]{sun2019model}
Wen Sun, Nan Jiang, Akshay Krishnamurthy, Alekh Agarwal, and John Langford.
\newblock Model-based rl in contextual decision processes: Pac bounds and
  exponential improvements over model-free approaches.
\newblock In \emph{Conference on Learning Theory}, pages 2898--2933, 2019.

\bibitem[Osband and Van~Roy(2014)]{osband2014model}
Ian Osband and Benjamin Van~Roy.
\newblock Model-based reinforcement learning and the eluder dimension.
\newblock In \emph{Advances in Neural Information Processing Systems}, pages
  1466--1474, 2014.

\bibitem[Dong et~al.(2020)Dong, Peng, Wang, and Zhou]{dong2020root}
Kefan Dong, Jian Peng, Yining Wang, and Yuan Zhou.
\newblock Root-n-regret for learning in markov decision processes with function
  approximation and low bellman rank.
\newblock In \emph{Conference on Learning Theory}, pages 1554--1557. PMLR,
  2020.

\bibitem[Yang et~al.(2020)Yang, Jin, Wang, Wang, and Jordan]{yang2020bridging}
Zhuoran Yang, Chi Jin, Zhaoran Wang, Mengdi Wang, and Michael~I Jordan.
\newblock Bridging exploration and general function approximation in
  reinforcement learning: Provably efficient kernel and neural value
  iterations.
\newblock \emph{arXiv preprint arXiv:2011.04622}, 2020.

\bibitem[Foster et~al.(2020)Foster, Rakhlin, Simchi-Levi, and
  Xu]{foster2020instance}
Dylan~J Foster, Alexander Rakhlin, David Simchi-Levi, and Yunzong Xu.
\newblock Instance-dependent complexity of contextual bandits and reinforcement
  learning: A disagreement-based perspective.
\newblock \emph{arXiv preprint arXiv:2010.03104}, 2020.

\bibitem[Du et~al.(2021)Du, Kakade, Lee, Lovett, Mahajan, Sun, and
  Wang]{du2021bilinear}
Simon~S Du, Sham~M Kakade, Jason~D Lee, Shachar Lovett, Gaurav Mahajan, Wen
  Sun, and Ruosong Wang.
\newblock Bilinear classes: A structural framework for provable generalization
  in rl.
\newblock \emph{arXiv preprint arXiv:2103.10897}, 2021.

\bibitem[Szepesv{\'a}ri(2010)]{szepesvari2010algorithms}
Csaba Szepesv{\'a}ri.
\newblock Algorithms for reinforcement learning.
\newblock \emph{Synthesis lectures on artificial intelligence and machine
  learning}, 4\penalty0 (1):\penalty0 1--103, 2010.

\bibitem[Puterman(2014)]{puterman2014markov}
Martin~L Puterman.
\newblock \emph{Markov decision processes: discrete stochastic dynamic
  programming}.
\newblock John Wiley \& Sons, 2014.

\bibitem[Weisz et~al.(2020)Weisz, Amortila, and
  Szepesv{\'a}ri]{weisz2020exponential}
Gellert Weisz, Philip Amortila, and Csaba Szepesv{\'a}ri.
\newblock Exponential lower bounds for planning in mdps with
  linearly-realizable optimal action-value functions.
\newblock \emph{arXiv preprint arXiv:2010.01374}, 2020.

\bibitem[Wainwright(2019)]{wainwright2019high}
Martin~J Wainwright.
\newblock \emph{High-dimensional statistics: A non-asymptotic viewpoint},
  volume~48.
\newblock Cambridge University Press, 2019.

\bibitem[Singh et~al.(2012)Singh, James, and Rudary]{singh2012predictive}
Satinder Singh, Michael James, and Matthew Rudary.
\newblock Predictive state representations: A new theory for modeling dynamical
  systems.
\newblock \emph{arXiv preprint arXiv:1207.4167}, 2012.

\bibitem[Antos et~al.(2008)Antos, Szepesv{\'a}ri, and Munos]{antos2008learning}
Andr{\'a}s Antos, Csaba Szepesv{\'a}ri, and R{\'e}mi Munos.
\newblock Learning near-optimal policies with bellman-residual minimization
  based fitted policy iteration and a single sample path.
\newblock \emph{Machine Learning}, 71\penalty0 (1):\penalty0 89--129, 2008.

\bibitem[Dann et~al.(2018)Dann, Jiang, Krishnamurthy, Agarwal, Langford, and
  Schapire]{dann2018oracle}
Christoph Dann, Nan Jiang, Akshay Krishnamurthy, Alekh Agarwal, John Langford,
  and Robert~E Schapire.
\newblock On oracle-efficient pac rl with rich observations.
\newblock In \emph{Advances in neural information processing systems}, pages
  1422--1432, 2018.

\bibitem[Agarwal et~al.(2014)Agarwal, Hsu, Kale, Langford, Li, and
  Schapire]{agarwal2014taming}
Alekh Agarwal, Daniel Hsu, Satyen Kale, John Langford, Lihong Li, and Robert
  Schapire.
\newblock Taming the monster: A fast and simple algorithm for contextual
  bandits.
\newblock In \emph{International Conference on Machine Learning}, pages
  1638--1646, 2014.

\end{thebibliography}

\newpage
\appendix


\pagebreak

\section{V-type BE  Dimension and Algorithms}
\label{app:BE-typeII}

The definition of Bellman rank, mentioned in Definition~\ref{def:bellman-rank-typeI} and Proposition~\ref{prop:bellman-bedim}, is slightly different from the original definition in \cite{jiang2017contextual}. We denote the former by {\bfseries Q-type} and the latter (the original definition) by {\bfseries V-type}. In this section we introduce V-type BE  Dimension as well as V-type variants of \golf\ and \olive. We show that similar results also hold for the V-type variants.

\begin{definition}[V-type Bellman rank]
The V-type Bellman rank is the minimum integer $d$ so that there exists $\phi_h: \Fcal\rightarrow\R^d$ and $\psi_h:\Fcal\rightarrow\R^d$ for each $h\in[H]$, such that for any $f,f'\in\Fcal$, the average V-type Bellman error
\begin{equation*}
\EcalII(f,\pi_{f'},h):=\E [ (f_h-\Tcal_h f_{h+1})(s_h,a_h) \mid s_h \sim \pi_{f'}, a_h \sim \pi_{f} ]	=\langle \phi_h(f),\psi_h(f')\rangle,
\end{equation*}
where $\|\phi_h(f)\|_2\cdot\|\psi_h(f')\|_2\le\zeta$, and $\zeta$ is the normalization parameter.
\end{definition}
The only difference between these two definitions is how we sample $a_h$. In the Q-type definition we have $a_h \sim \pi_{f'}$ (the roll-in policy), however in the V-type definition we have $a_h \sim \pi_{f}$ (the greedy policy of the function evaluated in the Bellman error) instead. It is worth mentioning that the Q-type and V-type bellman error coincide whenever $f=f'$; namely, $\Ecal(f,\pi_{f},h)=\EcalII(f,\pi_{f},h)$ for all $f \in \Fcal$. 

We can similarly define the V-type variant of BE  Dimension. At a high level, {\bfseries V-type BE  dimension} $\text{\bedimII}(\Fcal,\Pi,\epsilon)$  measures the complexity of finding a function in $\Fcal$ such that its expected Bellman error under any state distribution in $\Pi$ is smaller than $\epsilon$.

\begin{definition}[V-type BE  dimension]
Let $(I-\Tcal_h)V_{\Fcal} \subseteq (\Scal \rightarrow \mathbb{R})$ be the state-wise Bellman residual class of $\Fcal$ at step $h$ which is defined as 
$$
	(I-\Tcal_h)V_\Fcal := \big\{ s \mapsto (f_h-\Tcal_h f_{h+1})(s,\pi_{f_h}(s)) : f \in \Fcal\big\}.
$$ Let $\Pi=\{\Pi_h\}_{h=1}^{H}$ be a collection of $H$ probability measure families over $\Scal$. The {\bfseries V-type $\epsilon$-BE  dimension} of  $\Fcal$ with respect to $\Pi$ is  defined as
	$$
	\text{\bedimII}(\Fcal,\Pi,\epsilon) := 
	\max_{h\in[H]} \dedim\big((I-\Tcal_h)V_\Fcal,\Pi_h,\epsilon\big).
	$$
\end{definition}

\paragraph{Relation with low V-type Bellman rank} 
With slight abuse of notation, denote by $\Dcal_{\Fcal,h}$ the collection of all probability measures over $\Scal$ at the $h^{\rm th}$ step, which can be generated by rolling in with a greedy policy $\pi_f$ with $f\in\Fcal$.
Similar to Proposition~\ref{prop:bellman-bedim}, the following proposition claims  that the V-type BE  dimension of $\Fcal$ with respect to $\Dcal_{\Fcal}:=\{\Dcal_{\Fcal,h}\}_{h\in[H]}$ is always upper bounded by its V-type Bellman rank up to some logarithmic factor. 
\begin{proposition}[low V-type Bellman rank $\subset$ low V-type BE dimension]
\label{prop:bellman-bedim-typeII}
	If an MDP with function class $\cF$ has V-type Bellman rank $d$ with normalization parameter $\zeta$, then
\begin{equation*}
\text{\bedimII}(\Fcal,\Dcal_{\Fcal},\epsilon)\le \Ocal(1+d\log(1+\zeta /\epsilon)).
\end{equation*}
\end{proposition}
The proof of Proposition~\ref{prop:bellman-bedim-typeII} is almost the same as that of Proposition~\ref{prop:bellman-bedim} in Appendix~\ref{appendix:prop:bellman-bedim}.
We omit it here since the only modification is to replace Q-type Bellman rank with its V-type variant wherever it is used.

\subsection{Algorithm V-type \algx}

In this section we describe the V-type variant of \golf. The pseudocode is provided in Algorithm~\ref{alg:algxII}. Its only difference from the Q-type analogue is in Line~\ref{alg:algxII-diff1}: for each $h \in [H]$, we roll in with policy $\pi^k$ to sample $s_h$, and then instead of continuing following $\pi^k$ we take random action at step $h$.

\begin{algorithm}[t]
\caption{V-type \golf\ $(\Fcal,K,\beta)$} \label{alg:algxII}
 \begin{algorithmic}[1]
 \STATE \textbf{Initialize}:  $\Dcal_1,\dots,\Dcal_H\leftarrow \emptyset$, $\Bcal^0 \leftarrow \Fcal$.
 \FOR{\textbf{epoch} $k$ from $1$ to $K$} 
 \STATE \textbf{Choose policy} $\pi^k = \pi_{f^k}$, where $f^k = \argmax_{f \in \Bcal^{k-1}}f(s_1,\pit{f}(s_1))$. 
\FOR{\textbf{step} $h$ from $1$ to $H$}
\STATE \textbf{Collect} a tuple $(s_h,a_h,r_h,s_{h+1})$ by executing $\pi^k$ at step $1,\ldots,h-1$ and taking  action uniformly at random at step $h$. 
\label{alg:algxII-diff1}
\STATE \textbf{Augment} $\Dcal_h=\Dcal_h\cup\{(s_h,a_h,r_h,s_{h+1})\}$ for all $h\in[H]$.
\ENDFOR

\STATE \textbf{Update}
\vspace{-4mm}
\begin{equation*}
	\Bcal^{k}=\left\{ f \in \Fcal:\  \Lcal_{\Dcal_h}(f_h,f_{h+1}) \leq \inf_{g \in \Gcal_{h}} \Lcal_{\Dcal_h}(g,f_{h+1}) + \beta \ \mbox{for all }h\in[H]\right\},
\end{equation*}
\vspace{-3mm}
\hspace{+10mm}
\begin{equation*}
\mbox{where }	\Lcal_{\Dcal_h}(\xi_{h},\zeta_{h+1}) = \sum_{(s,a,r,s') \in \Dcal_h}[\xi_h(s,a)-r -\max_{a' \in \Acal} \zeta_{h+1}(s',a')]^2.
\end{equation*}
\ENDFOR
\vspace{-2mm}
\STATE \textbf{Output} $\pi^{\rm out}$ sampled uniformly at random from $\{\pi^k\}_{k=1}^{K}$.
 \end{algorithmic}
\end{algorithm}


Now we present the theoretical guarantee for Algorithm~\ref{alg:algxII}. Its proof is almost the same as that of Corollary \ref{cor:sample_golf} and  can be found in appendix~\ref{app:proof:algx-typeII}.

\begin{theorem}[V-type \golf]
\label{thm:main-typeII}
Under Assumption \ref{asp:realizability}, \ref{asp:G-completeness}, there exists an absolute constant $c$ such that for any given $\epsilon>0$, if we choose $\beta=c\log[KH\Ncal_{\Fcal\cup\Gcal}(\epsilon^2/(d|\Acal|H^2))]$, then with probability at least $0.99$, $\pi^{\rm out}$ is $\Ocal(\epsilon)$-optimal, if  
$$K\ge \Omega
\left( \frac{H^2 d |\Acal|}{\epsilon^2}\cdot\log\left[\Ncal_{\Fcal\cup\Gcal}\left(\frac{\epsilon^2}{H^2d|\Acal|}\right)\cdot \frac{Hd|\Acal|}{\epsilon}\right]\right),$$
where $d=\min_{\Pi\in\{\Dcal_{\dirac},\Dcal_{\Fcal}\}}
	 \text{\bedimII}\big(\Fcal,\Pi,{\epsilon}/{H}\big)$. 
	  \end{theorem}
Compared with Theorem \ref{thm:olive-typeII} (V-type \olive), 
Theorem \ref{thm:main-typeII} (V-type \algx) has the following two  advantages.
\begin{itemize}
	\item The sample complexity  in Theorem \ref{thm:main-typeII} depends linearly on the V-type BE-dimension while the dependence in Theorem \ref{thm:olive-typeII} is quadratic.
	\item Theorem \ref{thm:main-typeII} applies to RL problems of finite V-type BE  dimension with respect to either $\Dcal_{\Fcal}$ or $\Dcal_{\dirac}$. In comparison, Theorem \ref{thm:olive-typeII} provides no guarantee for the $\Dcal_{\dirac}$ case.
\end{itemize}

Finally, we comment that for the low Q-type  BE dimension  family, we provide both regret and sample complexity guarantees while for the  low V-type counterpart, we only derive sample complexity result due to the need of taking actions uniformly at random in Algorithm 
\ref{alg:olive-typeII} and  Algorithm 
\ref{alg:algxII}. 
\cite{dong2020root} propose an algorithm that can achieve $\sqrt{T}$-regret for  problems of low V-type  Bellman rank. It is an interesting open problem to study whether similar techniques can be adapted to the low V-type  BE dimension setting so that we can also obtain $\sqrt{T}$-regret.



\subsection{Algorithm V-type \olive}
In this section, we describe the original \olive\ (i.e., V-type \olive) proposed by \cite{jiang2017contextual}, and its theoretical guarantee in terms of V-type BE  dimension.

\begin{algorithm}[t]
\caption{V-type \textsc{Olive}\ $(\Fcal,\zeta_{\rm act},\zeta_{\rm elim},\nact,\nelim)$} 
\label{alg:olive-typeII}
\begin{algorithmic}[1]
\STATE \textbf{Initialize}: $\Bcal^0 \gets \Fcal$, $\Dcal_h \leftarrow \emptyset$ for all $h,k$.
\FOR{{\bf phase} $k=1,2,\ldots$} 
\STATE 
\textbf{Choose policy} $\pi^k = \pi_{f^k}$, where $f^k = \argmax_{f \in \Bcal^{k-1}}f(s_1,\pit{f}(s_1))$. 
\STATE  \textbf{Execute} $\pi^k$ for $\nact$ episodes and \emph{refresh} $\Dcal_h$ to include the fresh $(s_h,a_h,r_h,s_{h+1})$ tuples.
\STATE \textbf{Estimate} $\tildeberrII{f^k, \pi^k, h}$ for all $h\in[H]$, where 
\vspace{-3mm}
\begin{align*}
\tildeberrII{f^k, \pi^k, h} = \frac{1}{|\Dcal_h|} \sum_{(s,a,r,s') \in \Dcal_h}\left(f^k_h(s,a)-r - \max_{a'\in\Acal}f^k_{h+1}(s',a')\right).
\end{align*}
\vspace{-3mm}
\IF{$\sum_{h=1}^H \tildeberrII{f^k, \pi^k,h} \le H\zeta_{\rm act}$}
\STATE Terminate and output $\pi^k$. 
\ENDIF
\STATE Pick any $t \in [H]$ for which $\tildeberrII{f^k, \pi^k,t} > \zeta_{\rm act}$. 
\STATE  \textbf{Collect} $\nelim$ episodes by executing $\pi^k$ for step $1,\ldots,t-1$ and picking action uniform at random for step $t$. \emph{Refresh} $\Dcal_{h}$ to include the fresh $(s_h,a_h,r_h,s_{h+1})$ tuples. 
\label{olive-line:difference1}
\STATE \textbf{Estimate}  $\hatberrII{f, \pi^k, t}$ for all $f\in\Fcal$, where
\vspace{-2mm}
\begin{align*}
\hatberrII{f, \pi^k, h} = \frac{1}{|\Dcal_{h}|} \sum_{(s,a,r,s') \in \Dcal_{h}} \frac{\one[a = \pi_f(s)]}{1/|\Acal|} \left(f_h(s,a)-r - \max_{a'\in\Acal}f_{h+1}(s',a')\right).
\end{align*}
\vspace{-3mm}
\label{olive-line:difference2}
\STATE \textbf{Update} 
$\Bcal^{k} = \left\{f \in \Bcal^{k-1} : \left|\hatberrII{f,\pi^k,t} \right|\le \zeta_{\rm elim} \right\}.$
\ENDFOR
\end{algorithmic}
\end{algorithm}

The pseudocode is provided in Algorithm~\ref{alg:olive-typeII}. Its only difference from Algorithm~\ref{alg:olive} is  Line~\ref{olive-line:difference1}-\ref{olive-line:difference2}: note that V-type Bellman rank needs the action at step $t$ to be greedy with respect to the function $f$ instead of being picked by the roll-in policy $\pi^k$, so we choose action $a_{t}$ uniformly at random  and use the importance-weighted estimator to estimate the Bellman error for each $f$.

We have the following similar theoretical guarantee for Algorithm~\ref{alg:olive-typeII}. Its proof is almost the same as that of Theorem \ref{thm:olive} and can be found in Appendix~\ref{app:proof:olive-typeII}.
\begin{theorem}[V-type \olive]
\label{thm:olive-typeII}
Assume realizability (Assumption \ref{asp:realizability}) holds and $\Fcal$ is finite. There exists absolute constant $c$ such that if we choose 
$$\zeta_{\rm act}=\frac{2\epsilon}{H},\ \zetaelim=\frac{\epsilon}{2H\sqrt{d}},\ \nact=\frac{H^2 \iota}{\epsilon^2},\text{ and } \nelim=\frac{H^2d|\Acal|\log(|\Fcal|)\cdot \iota}{\epsilon^2}
$$
 where $d=\text{\bedimII}\big(\Fcal,\Dcal_{\Fcal},{\epsilon}/{H}\big)$ and  $\iota=c\log[Hd|\Acal|/\delta\epsilon]$, 
 then with probability at least $1-\delta$, Algorithm \ref{alg:olive-typeII} will output an $\Ocal(\epsilon)$-optimal policy using at most  $\Ocal({H^3d^2|\Acal|\log(|\Fcal|)\cdot \iota}/{\epsilon^2})$ episodes.
\end{theorem}
For problems with Bellman rank $d$ and finite function class $\Fcal$, Theorem \ref{thm:olive-typeII} together with Proposition \ref{prop:bellman-bedim-typeII} 
guarantees $\tilde{\Ocal}(H^3d^2|\Acal|\log(|\Fcal|)/\epsilon^2)$ samples suffice for finding an $\epsilon$-optimal policy, which matches the result in \cite{jiang2017contextual}.
 For function class $\Fcal$ of infinite cardinality but with finite covering number, we can first compute an $\Ocal(\zetaelim)$-cover of $\Fcal$, which we denote as $\Zcal_{\rho}$, and then run Algorithm \ref{alg:olive-typeII} on $\Zcal_{\rho}$.
 By following almost the same arguments in the proof of Theorem \ref{thm:olive-typeII} (the only difference is to replace $Q^\star$ by its proxy in $\Zcal_\rho$), we can show Algorithm \ref{alg:olive-typeII} will output an $\Ocal(\epsilon)$-optimal policy using at most $\tilde\Omega({H^3d^2|\Acal|\log(N)}/{\epsilon^2})$ episodes where  $N=\Ncal_\Fcal(\Ocal(\zetaelim))$.




\subsection{Discussions on Q-type versus V-type}
\label{discuss:QV}

In this paper, we have introduced two complementary definitions of Bellman rank: Q-type Bellman rank and V-type Bellman rank.  
And we prove they are upper bounds for Q-type and V-type BE dimension, respectively. 
Here, we want to emphasize that both Q-type and V-type Bellman rank have their own advantages. Specifically, the Q-type version has the following strengths.
\begin{enumerate}
	\item There are natural RL problems whose Q-type Bellman rank is small, while their V-type Bellman rank is very large, e.g., the linear function approximation setting studied in in \cite{zanette2020learning}.
	\item All the existing sample complexity results for the V-type cases scale linearly with respect to the number of actions, while those for the Q-type cases are independent of the number of actions. Therefore, for control problems such as Linear Quadratic Regulator (LQR), which has both small Q-type and V-type Bellman rank but infinite number of actions, the notion of Q-type is more suitable.
\end{enumerate}

 On the other hand, there are problems that naturally induce low V-type Bellman rank but have large Q-type Bellman rank, e.g., reactive POMDPs.


\section{Examples}
\label{app:examples}

In this section, we introduce  examples with low BE dimension. 
We will start with linear models and their variants, then introduce kernel MDPs, and finally present kernel reactive POMDPs which  have low BE dimension, but possibly large Bellman rank and large Eluder dimension.
All the proofs for this section are deferred to Appendix \ref{app:proof-examples}.

\subsection{Linear models and their variants}

In this subsection, we review problems with linear structure in ascending order of generality.
We start with the definition of linear MDPs \cite[e.g.,][]{jin2020provably}.
\begin{definition}[Linear MDPs]
We say an MDP is linear of dimension $d$ if for each $h\in[H]$, there exists feature mappings $\phi_h: \Scal\times\Acal\rightarrow\R^d$, and $d$ unknown signed measures $\psi_h=(\psi_h^{(1)},\ldots,\psi_h^{(d)})$ over $\Scal$, and an unknown vector $\theta_h^r\in\R^d$,
such that $\Pr_h(\cdot\mid s,a)=\phi_h(s,a)\trans\psi_h(\cdot)$ and $r_h(s,a)=\phi_h(s,a)\trans\theta_h^r$ for all $(s,a)\in\Scal\times\Acal$. 
\end{definition}
We remark that existing works \cite[e.g.,][]{jin2020provably} usually assumxe $\phi$ is \emph{known} to the learner.
Next, we review a more general setting---the linear completeness setting \cite[e.g.,][]{zanette2020learning}.
\begin{definition}[Linear completeness setting]
We say an MDP is in the linear completeness setting of dimension $d$, if there exists a feature mapping $\phi_h: \Scal\times\Acal\rightarrow\R^d$, such that for the linear function class $\cF_h = \{ \phi_h(\cdot)\trans \theta ~|~ \theta \in \R^d\}$, both Assumption \ref{asp:realizability} and \ref{asp:completeness} are satisfied.
\end{definition}
We make three comments here. Firstly, we note that linear MDPs automatically satisfy both linear realizability and linear completeness assumptions, therefore are special cases of the linear completeness setting with the same ambient dimension. Secondly, only assuming linear realizability but without completeness is insufficient for sample-efficient learning (see exponential  lower bounds in \cite{weisz2020exponential}). 
Finally, as mentioned in Appendix \ref{discuss:QV}, though MDPs in the linear completeness setting have low Q-type Bellman rank, their V-type Bellman rank can be arbitrarily large.

Finally, we review the generalized linear completeness setting \citep{wang2019optimism}, which generalizes the linear completeness setting by adding nonlinearity.

\begin{definition}[Generalized linear completeness setting]
We say an MDP is in the generalized linear completeness setting of dimension $d$, if there exists a feature mapping $\phi_h: \Scal\times\Acal\rightarrow\R^d$, and a link function $\sigma$, such that for the generalized linear function class $\cF_h = \{ \sigma(\phi_h(\cdot)\trans \theta) ~|~ \theta \in \R^d\}$, both Assumption \ref{asp:realizability} and \ref{asp:completeness} are satisfied, and the link function is strictly monotone, i.e., there exist $0<c_1<c_2<\infty$ such that $\sigma'(x)\in[c_1,c_2]$ for all $x$.
\end{definition}
One can directly verify by definition that when we choose link function $\sigma(x)=x$ in the generalized linear completeness setting, it will reduce to the standard linear version. 
Besides, it is known \citep{russo2013eluder} the generalized linear completeness setting is a special case of low Eluder dimension, thus belonging to the low BE dimension family. Finally, we comment that despite the linear completeness setting belongs to the low Bellman rank family, the generalized version does not because of the possible nonlinearity of the link function.

\subsection{Effective dimension and  kernel MDPs}\label{subsec:eff-dim}

In this subsection, we introduce the notion of effective dimension. With this notion, we prove a useful proposition that any linear kernel function class with low effective dimension also has low Eluder dimension. This proposition directly implies that kernel MDPs are special cases of low Eluder dimension, which are also special cases of low BE dimension.

\paragraph{Effective dimension}
We start with the definition of effective dimension for a set, which is also known as critical information gain in \citet{du2021bilinear}.
\begin{definition}[$\epsilon$-effective dimension of a set]
	The $\epsilon$-effective dimension of a set $\Xcal$  is the minimum integer $d_{{\rm eff}}(\Xcal,\epsilon)=n$ such that
\begin{equation}
\sup_{x_1,\ldots,x_n\in \Xcal}	\frac{1}{n} \log\det\left(\mathrm I + \frac{1}{\epsilon^2} \sum_{i=1}^{n} x_i x_i^\top \right) \le e^{-1}.
\end{equation}
\end{definition}

Based on this definition, we can also define the effective dimension of a function class.
\begin{definition}[$\epsilon$-effective dimension of a function class]
Given a function class $\Fcal$ defined on $\Xcal$, its $\epsilon$-effective dimension $d_{{\rm eff}}(\Fcal,\epsilon)=n$ is the minimum integer $n$ such that there exists a separable Hilbert space $\cH$ and a mapping $\phi:\Xcal\rightarrow\cH$ so that 
\begin{itemize}
	\item for every $f\in\Fcal$ there exists $\theta_f\in B_\cH(1)$ satisfying $f(x) = \langle \theta_f, \phi(x)\rangle_\cH$ for all $x\in\Xcal$,
	\item $d_{{\rm eff}}(\phi(\Xcal),\epsilon)=n$ where $\phi(\Xcal)= \{ \phi(x):\ x\in\Xcal\}$.
\end{itemize}
\end{definition}

%

The following proposition shows that the Eluder dimension of any function class is always upper bounded by its effective dimension. 
 
\begin{proposition}[low effective dimension $\subset$ low Eluder dimension]\label{prop:eff-eluder}
For any function class $\Fcal$ and domain $\Xcal$, we have 
$$
\dim_{\rm E}(\Fcal,\epsilon)\le  \dim_{\rm eff}(\Fcal,\epsilon/2).
$$	
\end{proposition}

On the other hand, we remark that effective dimension requires the existence of a benign linear structure in certain Hilbert spaces. In constrast, Eluder dimension does not require such conditions. 
Therefore, the function class of low Eluder dimension is more general than the function class of low effective dimension.

\paragraph{Kernel MDPs}
Now, we are ready to define kernel MDPs and prove it is a subclass of low Eluder dimension. 
\begin{definition}[Kernel MDPs]
	In a kernel MDP of effective dimension $d(\epsilon)$, for each step $h\in[H]$, there exist feature mappings $\phi_h: \Scal\times\Acal\rightarrow\cH$ and  $\psi_h: \Scal\rightarrow\cH$ where $\cH$ is a separable Hilbert space, so that the transition measure can be represented as the inner product of features, i.e.,  $\Pr_h(s'\mid s,a)=\langle \phi_h(s,a),\psi_h(s')\rangle_\cH$.  Besides, the reward function is linear in  $\phi$, i.e., $r_h(s,a)=\langle\phi_h(s,a),\theta_h^r\rangle_\cH$ for some  $\theta_h^r\in\cH$. Here, $\phi$ is known to the learner while $\psi$ and $\theta^r$ are \emph{unknown}. Moreover, a kernel MDP satisfies the following regularization conditions: for all $h$
	\begin{itemize}
		\item $\|\theta_h^r\|_\cH\le 1$ and $\|\phi_h(s,a)\|_\cH \le 1$ for all $s,a$.
		\item $\| \sum_{s\in\Scal} \Vcal(s)\psi_h({s})\|_\cH \le 1$ for any function $\Vcal:\Scal\rightarrow [0,1]$.
		\item $ \dim_{\rm eff}(\Xcal_h,\epsilon)\le d(\epsilon)$ for  all $h$ and $\epsilon$, where 
		$\Xcal_h = \{ \phi_h(s,a):~(s,a)\in\Scal\times\Acal\}$.
	\end{itemize}
\end{definition}

In order to learn kernel MDPs, we need to construct a proper function class $\Fcal$.
Formally, for each $h\in[H]$, we choose $\Fcal_h=\{\phi_h(\cdot,\cdot)\trans\theta~\mid~\theta\in B_\cH(H+1-h)\}$. 
One can easily verify $\Fcal$ satisfies both realizability and completeness by following the same arguments as in linear MDPs \citep{jin2020provably}. 
In order to apply \golf\ or \olive, we also need to show it has low BE dimension and bounded log-covering number. 
Below, we prove in sequence that $\Fcal$ has  low Eluder dimension and low log-covering number. Therefore, kernel MDPs fall into our low BE dimension framework.

\begin{proposition}[kernel MDPs $\subset$ low Eluder dimension]\label{prop:kernelMDP-eluder}
Let $\Mcal$ be a kernel MDP of effective dimension $d(\epsilon)$, then
$$
\dim_{\rm E}(\Fcal,\epsilon)\le d(\epsilon/2H).
$$	
\end{proposition}
Proposition \ref{prop:kernelMDP-eluder} follows directly from Proposition \ref{prop:eff-eluder} by rescaling the parameters. 
Utilizing Proposition \ref{prop:kernelMDP-eluder}, we can further prove the log-covering number of $\Fcal$ is also upper bounded by the effective dimension of the kernel MDP up to some logarithmic factor.

\begin{proposition}[bounded covering number]\label{prop:kernelMDP-cover}
Let $\Mcal$ be a kernel MDP of effective dimension $d(\epsilon)$, then
$$
\log\Ncal_\Fcal(\epsilon) \le \Ocal\big(Hd(\epsilon) \cdot\log(1+d(\epsilon)H/\epsilon)\big).
$$
\end{proposition}

\subsection{Effective Bellman rank and kernel reactive POMDPs}
\label{app:effective-dim}

To begin with, we introduce the definition of effective Bellman rank and prove that it is always an upper bound for BE dimension. We will see effective Bellman rank serves as a useful tool for controlling the BE dimension of the example discussed in this section---kernel reactive POMDPs.

\paragraph{Q-type effective Bellman rank}
We start with Q-type $\epsilon$-effective Bellman rank which is simply the $\epsilon$-effective dimension of a special feature set.
\begin{definition}[Q-type $\epsilon$-effective Bellman rank]\label{def:effective-bellman-rank-typeI}
The  Q-type  $\epsilon$-effective Bellman rank is the minimum integer $d$ so that 
\begin{itemize}
\item 	There exists $\phi_h: \Fcal\rightarrow\cH$ and $\psi_h:\Fcal\rightarrow\cH$ for each $h\in[H]$ where $\cH$ is a separable Hilbert space, such that for any $f,f'\in\Fcal$, the average Bellman error
\begin{equation*}
\Ecal(f,\pi_{f'},h):=\E_{\pi_{f'}} [ (f_h-\Tcal_h f_{h+1})(s_h,a_h)]	=\langle \phi_h(f),\psi_h(f')\rangle_\cH
\end{equation*}
where $\|\phi_h(f)\|_\cH\le \zeta$,  and $\zeta$ is the normalization parameter.
\item $d = \max_{h\in[H]}d_{\rm eff}(\Xcal_h(\psi,\Fcal),\epsilon/\zeta)$ where $\Xcal_h(\psi,\Fcal) = \{ \psi_h(f_h):\ f_h\in\Fcal_h\}$. 
\end{itemize}
\end{definition}
One can easily verify that when $\cH$ is a finite-dimensional Euclidean space, the $\epsilon$-effective Bellman rank is always upper bounded by the original Bellman rank up to a logarithmic factor in $\zeta$ and $\epsilon^{-1}$. Moreover, the effective Bellman rank can be much smaller than the original Bellman rank if the induced feature set $\{\Xcal_h(\psi,\Fcal)\}_{h\in[H]}$ approximately lies in a low-dimensional linear subspace.
Therefore, effective Bellman rank can be viewed as a strict generalization of the original version.

\begin{proposition}[low  Q-type  effective Bellman rank $\subset$ low  Q-type  BE dimension]
\label{prop:effective-bellman-bedim}
Suppose function class $\cF$ has  Q-type  $\epsilon$-effective Bellman rank $d$, then
\begin{equation*}
\BEdim(\Fcal,\Dcal_{\Fcal},\epsilon)\le d.
\end{equation*}
\end{proposition}
Proposition \ref{prop:effective-bellman-bedim} claims that problems with low  Q-type  effective Bellman rank also have low  Q-type  BE dimension.

\paragraph{V-type effective Bellman rank} 	We can similarly define the V-type variant of effective Bellman rank, and prove it is always an upper bound for   V-type  BE dimension.
\begin{definition}[V-type $\epsilon$-effective Bellman rank]\label{def:effective-bellman-rank-typeII}
The  V-type  $\epsilon$-effective Bellman rank is the minimum integer $d$ so that 
\begin{itemize}
\item 	There exists $\phi_h: \Fcal\rightarrow\cH$ and $\psi_h:\Fcal\rightarrow\cH$ for each $h\in[H]$ where $\cH$ is a separable Hilbert space, such that for any $f,f'\in\Fcal$, the average Bellman error
\begin{equation*}
\EcalII(f,\pi_{f'},h):=\E [ (f_h-\Tcal_h f_{h+1})(s_h,a_h) \mid s_h \sim \pi_{f'}, a_h \sim \pi_{f} ]	=\langle \phi_h(f),\psi_h(f')\rangle_\cH
\end{equation*}
where $\|\phi_h(f)\|_\cH\le \zeta$,  and $\zeta$ is the normalization parameter.
\item $d = \max_{h\in[H]}d_{\rm eff}(\Xcal_h(\psi,\Fcal),\epsilon/\zeta)$ where $\Xcal_h(\psi,\Fcal) = \{ \psi_h(f_h):\ f_h\in\Fcal_h\}$. 
\end{itemize}
\end{definition}

\begin{proposition}[low  V-type  effective Bellman rank $\subset$ low  V-type  BE dimension]
\label{prop:effective-bellman-bedim-V}
Suppose  function class $\cF$ has  V-type  $\epsilon$-effective Bellman rank $d$, then
\begin{equation*}
\text{\bedimII}(\Fcal,\Dcal_{\Fcal},\epsilon)\le d.
\end{equation*}
\end{proposition}

The proof of Proposition~\ref{prop:effective-bellman-bedim-V} is almost the same as that of Proposition~\ref{prop:effective-bellman-bedim}.
We omit it since the only modification is to replace Q-type effective Bellman rank with its V-type variant wherever it is used.	
	
 We want to briefly comment that the majority of examples introduced in \cite{du2021bilinear} have low effective Bellman rank. For example, low occupancy complexity,  linear $Q^*/V^*$, linear Bellman complete and $Q^*$ state aggregation have low Q-type effective Bellman rank. And the feature selection problem has low V-type Bellman rank.

\paragraph{Kernel reactive POMDPs}
We start with the definition of POMDPs. 
	A POMDP is defined by a tuple $(\Scal,\Acal,\Ocal,\T,\mathbb{O},r,H)$ where $\Scal$ denotes the set of hidden states, $\Acal$ denotes the set of actions, $\Ocal$ denotes the set of observations, $\T$ denotes the transition measure, $\mathbb{O}$ denotes the emission measure, $r=\{r_h\}_{h=1}^H$ denotes the collections of reward functions, and $H$ denotes the length of each episode. At the beginning of each episode, the agent always starts from a fixed initial state. At each step $h\in[H]$, after reaching $s_h$, the agent will observe 
	$o_h\sim \mathbb{O}_h(\cdot\mid s_h)$. Then the agent picks action $a_h$, 
	 receives $r_h(o_h,a_h)$ and transits to $s_{h+1}\sim \T_h(\cdot\mid s_h,a_h)$. In POMDPs,  the agent can never directly observe the states $s_{1:H}$. It can only observe $o_{1:H}$ and $r_{1:H}$. Now we are ready to formally define kernel reactive POMDPs.
	
\begin{definition}[Kernel reactive POMDPs]
	A kernel reactive POMDP is a POMDP that additionally satisfies the following two conditions 
	\begin{itemize}
		\item For each $h\in[H]$, there exist mappings $\phi_h:\Scal\times\Acal\rightarrow\cH$ and $\psi_h:\Scal\rightarrow\cH$
		 where $\cH$ is a separable Hilbert space, such that $\T_h(s'\mid s,a)=\langle \phi_h(s,a),\psi_h(s')\rangle_\cH$ for all $s',a,s$. Moreover, for any function $\Vcal:\Scal\rightarrow [0,1]$, $\| \sum_{s'\in\Scal} \Vcal(s')\psi_h({s'})\|_\cH \le 1$.
		\item (Reactiveness) The optimal action-value function $Q^*$ only depends on the current observation and action, i.e., for each $h\in[H]$, there exists function $f^*_h:\Ocal\times\Acal\rightarrow[0,1]$ such that for all $\tau_h=[o_1,a_1,r_1,\ldots,o_h]$ and $a_h$
		$$
		Q^*_h(\tau_h,a_h) = f^*_h(o_{h},a_h).
		$$
	\end{itemize}
\end{definition}

The following proposition shows that when a kernel reactive POMDP has low effective dimension, it also has low V-type BE dimension. 

\begin{proposition}[kernel reactive POMDPs $\subset$ low V-type BE dimension]
\label{ex:pomdps}
Any kernel reactive POMDP and function class $\Fcal\subseteq(\Ocal\times\Acal\rightarrow[0,1])$ satisfy 
\begin{equation*}
\text{\bedimII}(\Fcal,\Dcal_{\Fcal},\epsilon)\le \max_{h\in[H]} d_{\rm eff}(\Xcal_h,\epsilon/2),
\end{equation*}
 where  $\Xcal_h = \{\E_{\pi_f}[\phi_{h}(s_{h},a_{h})]: \ f\in\Fcal\}$.
\end{proposition}

We comment that when $\cH$ \emph{approximately} aligns with a low-dimensional linear subspace, the V-type effective Bellman rank in Proposition \ref{ex:pomdps} will also be low. However, the Eluder dimension of $\Fcal$ can be arbitrarily large because we basically pose no structural assumption on $\Fcal$. Besides, its V/Q-type original Bellman rank can also be arbitrarily large, because $\cH$ may be infinite-dimensional and the observation set $\Ocal$ may be exponentially large. If we additionally assume $\Fcal$  satisfies realizability ($f^*\in\Fcal$), then we can apply V-type \olive\ and obtain polynomial sample-complexity guarantee.

\section{Proofs for  BE Dimension}

In this section, we provide formal proofs for the results stated in Section \ref{sec:DEdim}.

\subsection{Proof of Proposition \ref{prop:bellman-bedim}}
\label{appendix:prop:bellman-bedim}
The proof is basically the same as that of Example 3 in \cite{russo2013eluder} with minor modification.
\begin{proof}
	Without loss of generality, assume $\max\{\|\phi_h(f)\|_2,\|\psi_h(f)\|_2\}\le\sqrt{\zeta}$, otherwise we can satisfy this assumption by rescaling the feature mappings. 
	Assume there exists $h\in[H]$  such that $\dedim((I-\Tcal_h)\Fcal,\Dcal_{\Fcal,h},\epsilon)\ge m$.
	Let $\mu_1,\ldots,\mu_m\in\Dcal_{\Fcal,h}$ be a  an $\epsilon$-independent sequence with respect to $(I-\Tcal_h)\Fcal$. By Definition \ref{def:ind_dist}, there exists $f^1,\ldots,f^m$ such that for all $i\in[m]$, $\sqrt{\sum_{t=1}^{i-1}(\E_{\mu_t}[f^i_h-\Tcal_h f_{h+1}^i])^2}\le \epsilon$ and $|\E_{\mu_i}[f^i_h-\Tcal_h f_{h+1}^i] |>\epsilon$.
	Since $\mu_1,\ldots,\mu_n\in\Dcal_{\Fcal,h}$, there exist $g^1,\dots,g^n\in\Fcal$ so that $\mu_i$ is generated by executing $\pi_{g^i}$ for all $i\in[n]$.

	By  the definition of Bellman rank, this is equivalent to: for all $i\in[m]$, $\sqrt{\sum_{t=1}^{i-1}(\langle \phi_h(g^i),\psi_h(f^t)\rangle)^2}\le \epsilon$ and $| \langle \phi_h(g^i),\psi_h(f^i)\rangle|>\epsilon$.

	For notational simplicity, define $\x_i = \phi_h(g^i)$, $\z_i = \psi_h(f^i)$ and $\V_i= \sum_{t=1}^{i-1} \z_t\z_t\trans+\frac{\epsilon^2}{\zeta}\cdot\I$. 
	The previous argument directly implies: for all $i\in[m]$, $\|\x_i\|_{\V_i} \le \sqrt{2}\epsilon$ and $\| \x_i \|_{\V_i}\cdot\| \z_i \|_{\V_i^{-1}}>\epsilon$.
	Therefore, we have $\| \z_i \|_{\V_i^{-1}}\ge \frac{1}{\sqrt{2}}$.
	
	By the matrix determinant lemma, 
	$$
	\det [\V_m] = \det [\V_{m-1}] ( 1+ \| \z_m \|_{\V_m^{-1}}^2) \ge \frac{3}{2}\det [\V_{m-1}] \ge \ldots \ge \det[\frac{\epsilon^2}{\zeta}\cdot\I] (\frac{3}{2})^{m-1} = 
	(\frac{\epsilon^2}{\zeta})^{d} (\frac{3}{2})^{m-1}.
	$$
On the other hand, 
$$
\det [\V_m]  \le (\frac{{\rm trace}[\V_m]}{d})^d \le 
(\frac{\zeta(m-1)}{d}+\frac{\epsilon^2}{\zeta})^d.
$$
Therefore, we obtain
$$
(\frac{3}{2})^{m-1}
\le 
(\frac{\zeta^2(m-1)}{d\epsilon^2}+1)^d.
$$
Take logarithm on both sides,
$$
m \le 4 \left[1+  d\log(\frac{\zeta^2(m-1)}{d\epsilon^2}+1)\right],
$$
which, by simple calculation, implies 
$$
m \le \Ocal\left(1+d\log(\frac{\zeta^2}{\epsilon^2}+1)\right). 
\vspace{-8mm}
$$
\end{proof}

\subsection{Proof of Proposition \ref{prop:eluder-bedim}}
\label{appendix:prop:eluder-bedim}
\begin{proof}
	Assume $\delta_{z_1},\ldots,\delta_{z_m}$ is an $\epsilon$-independent   sequence of distributions with respect to $(I-\Tcal_h)\Fcal$, where $\delta_{z_i}\in\Dcal_{\dirac}$. 
	By Definition \ref{def:ind_dist}, there exist functions $f^1,\ldots,f^m\in\Fcal$ such that for all $i\in[m]$, we have
	$|(f^i_h-\Tcal_h f^i_{h+1})(z_i)| > \epsilon$ and $\sqrt{\sum_{t=1}^{i-1}|(f^i_h-\Tcal_h f^i_{h+1})(z_t)|^2}\le \epsilon$. 
	Define $g^i_h=\Tcal_{h} f^i_{h+1}$. Note that $g^i_h\in\Fcal_h$ because $\Tcal_h\Fcal_{h+1}\subset\Fcal_h$.  
	Therefore, we have for all $i\in[m]$,
	$|(f^i_h - g^i_h)(z_i)| > \epsilon$ and $\sqrt{\sum_{t=1}^{i-1}|(f^i_h-g^i_h)(z_t)|^2}\le \epsilon$ with $f^i_h,g_h^i\in\Fcal_h$. 
	By Definition \ref{def:ind_points} and \ref{def:eluder}, this implies
	$\dim_{\rm E}(\Fcal_h,\epsilon)\ge m$, which completes the proof.
\end{proof}

\subsection{Proof of Proposition \ref{lem:lowerbound}}
\label{app:proof-eluder-lowerbound}
\begin{proof}
For any $m\in\N^+$, 
denote by $e_1,\ldots,e_m$ the basis vectors in $\R^m$, and
	consider the following linear bandits ($|\Scal|=H=1$) problem.
	\begin{itemize}
		\item The action set $\Acal=\{a_i=(1;e_i)\in\R^{m+1}: \ i\in[m]\}$.
		\item The function set $\Fcal_1=\{ f_{\theta_i}(a)=a\trans \theta_i:\  \theta_i=(1;e_i),\ i\in[m]\}$.
		 \item The reward function is always zero, i.e., $r\equiv 0$.  
	\end{itemize}
\paragraph{Eluder dimension}
For any $\epsilon\in(0,1]$, $a_1,\ldots,a_{m-1}$ is an $\epsilon$-independent sequence of points because: (a) for any $t\in[m-1]$,
$\sum^{t-1}_{i=1}(f_{\theta_t}(a_i)-f_{\theta_{t+1}}(a_i))^2=0$; (b) for any $t\in[m-1]$,
$f_{\theta_t}(a_t)-f_{\theta_{t+1}}(a_t)=1\ge \epsilon$. Therefore, $\min_{h\in[H]}\dim_{\rm E}(\Fcal_h,\epsilon) = \dim_{\rm E}(\Fcal_1,\epsilon) \ge m-1$.

\paragraph{Bellman rank} It is direct to see the Bellman residual matrix is $\Ecal:=\Theta\trans\Theta\in\R^{m\times m}$ with rank $m$, where $\Theta=[\theta_1,\theta_2,\ldots,\theta_m]$. As a result, the Bellman rank is at least $m$.

\paragraph{BE dimension}
First, note in this setting $(I-\Tcal_1)\Fcal$  is simply $\Fcal_1$ (because $\Fcal_2=\{0\}$ and $r\equiv0$), and  $\Dcal_\Fcal$ coincides with $\Dcal_\Delta$,  so it suffices to show $\dedim(\Fcal_1,\Dcal_{\dirac},\epsilon)\le 5$.

Assume $\dedim(\Fcal_1,\Dcal_{\dirac},\epsilon)=k$.
	Then there exist $q_1,\ldots,q_k\in\Acal$ and $w_1,\ldots,w_k\in\Acal$ such that for all $t\in[k]$, $\sqrt{\sum_{i=1}^{t-1}(\langle q_t,w_i\rangle)^{2}}\le \epsilon$ 
	and $| \langle q_t,w_t\rangle|>\epsilon$. By simple calculation, we have $q_i\trans w_j\in[1,2]$ for all $i,j\in[k]$. 
	Therefore, if $\epsilon>2$, then $k=0$ because $|\langle q_t,w_t\rangle|\le 2$; if $\epsilon\le 2$, then $k\le 5$ because $\sqrt{k-1}\le \sqrt{\sum_{i=1}^{k-1}(\langle q_k,w_i\rangle)^2}\le \epsilon$. 
\end{proof}

%
%
%

\section{Proofs for \golf}
\label{appendix:lcgo}

In this section, we provide formal proofs for the results stated in Section \ref{sec:lcgo}.

\subsection{Proof of Theorem \ref{thm:main}}
\label{appendix:golf-fullproof}

We start the proof with the following two lemmas.
The first lemma  shows that with high probability any function in the confidence set has low Bellman-error over the collected datasets $\Dcal_1,\ldots,\Dcal_H$ as well as the distributions from which $\Dcal_1,\ldots,\Dcal_H$ are sampled.  
\begin{lemma}\label{lem:confiset-2}
Let $\rho>0$ be an arbitrary fixed number.
If we choose $\beta=c \big( \log[KH\Ncal_{\Fcal\cup\Gcal}(\rho)/\delta]+K\rho\big)$ with some large absolute constant $c$ in Algorithm \ref{alg:algx}, then
	with probability at least $1-\delta$, for all $(k,h)\in[K]\times[H]$, we have 
	\begin{enumerate}[label=(\alph*)]
		\item $\sum_{i=1}^{k-1}  \E [ \paren{f_h^k(s_h,a_h) - (\cT f_{h+1}^k)(s_h,a_h) }^2\mid s_h,a_h\sim \pi^i]
		{\le} \mathcal{O} ( \beta)$.
		\item $\sum_{i=1}^{k-1}  \paren{f^k_h(s_h^i,a_h^i) - (\cT f_{h+1}^k)(s_h^i,a_h^i) }^2
		{\le} \mathcal{O} ( \beta)$,
	\end{enumerate}
	where $(s_1^i,a_1^i,\ldots, s_H^i,a_H^i,s_{H+1}^i)$ denotes the trajectory  sampled by following $\pi^i$ in the $i^{\rm th}$ episode. 
\end{lemma}

The second lemma guarantees that the optimal value function is inside the confidence with high probability. As a result, the selected value function $f^k$ in each iteration shall be an upper bound of $Q^\star$ with high probability.
\begin{lemma}\label{lem:confiset-1}
Under the same condition of Lemma \ref{lem:confiset-2}, with probability at least $1-\delta$, we have 
	 $Q^\star\in \Bcal^k$ for all $k\in[K]$.
		\end{lemma}
The proof of Lemma \ref{lem:confiset-2} and \ref{lem:confiset-1} relies on standard martingale concentration (e.g. Freedman's inequality) and can be found in Appendix \ref{appendix:concentration}.

\paragraph{Step 1. Bounding the regret by Bellman error}
By Lemma \ref{lem:confiset-1}, we can upper bound the cumulative regret by the summation of Bellman error with probability at least $1-\delta$:
\begin{equation}\label{eq:algx-step1}
	\begin{aligned}
	\sum_{k=1}^K \paren{V_1^\star(s_1) - V^{\pi^k}_1(s_1)}
	\le \sum_{k=1}^K \paren{\max_{a}f^k_1(s_1,a) - V^{\pi^k}_1(s_1)}
	\overset{(i)}{=}
	\sum_{k=1}^{K} \sum_{h=1}^{H} \Ecal(f^{k},\pi^{k},h),
\end{aligned}
\end{equation}
where $(i)$ follows from standard policy loss decomposition  (e.g. Lemma 1 in \cite{jiang2017contextual}). 

\paragraph{Step 2. Bounding cumulative Bellman error using \deber \ dimension}
Next, we focus on a fixed step $h$ and bound the cumulative Bellman error $\sum_{k=1}^{K} \Ecal(f^{k},\pi^{k},h)$ using Lemma \ref{lem:confiset-2}.
To proceed, we need the following lemma to control the accumulating rate of Bellman error.
\begin{lemma}\label{lem:de-regret}
Given a function class $\Phi$ defined on $\Xcal$ with $|\phi(x)|\le C$ for all $(g,x)\in\Phi\times\Xcal$, and a family of probability measures $\Pi$ over $\Xcal$. 
	Suppose sequence $\{\phi_k\}_{k=1}^{K}\subset \Phi$ and $\{\mu_k\}_{k=1}^{K}\subset\Pi$ satisfy that for all $k\in[K]$,
	$\sum_{t=1}^{k-1} (\E_{\mu_t} [\phi_k])^2 \le \beta$. Then for all $k\in[K]$ and $\omega>0$,
	$$
	\sum_{t=1}^{k} |\E_{\mu_t} [\phi_t]| \le \Ocal\left(\sqrt{\dedim (\Phi,\Pi,\omega)\beta k}+\min\{k,\dedim (\Phi,\Pi,\omega)\}C +k\omega\right).
	$$
\end{lemma}
Lemma \ref{lem:de-regret} is a simple modification of Lemma 2 in \cite{russo2013eluder} and its proof can be found in Appendix \ref{appendix:de-regret}. 
 We provide two ways to apply Lemma \ref{lem:de-regret}, which can produce regret bounds in term of two different complexity measures. 
If we invoke  Lemma \ref{lem:confiset-2} (a) and Lemma \ref{lem:de-regret} with
\begin{equation*}
\left\{
\begin{aligned}
& \rho =\frac{1}{K},\ \omega=  \sqrt{\frac{1}{K}},\ C=1, \\
&	\Xcal = \Scal\times\Acal, \ \Phi=(I-\Tcal_h)\Fcal, \ \Pi=\Dcal_{\Fcal,h},\\
& \phi_k = f^k_h-\Tcal_h f^k_{h+1} \mbox{ and } \mu_k = \Pr^{\pi^{k}}(s_h=\cdot,a_h=\cdot),
\end{aligned}
\right.
\end{equation*}
we obtain
\begin{equation}\label{eq:algx-step2}
	\sum_{t=1}^{k} \Ecal(f^{t},\pi^t,h) \le \Ocal\left(\sqrt{ k\cdot\text{\bedim}(\Fcal,\Dcal_{\Fcal},\sqrt{1/K}) \log[KH\Ncal_{\Fcal\cup\Gcal}({1}/{K})/\delta]} \right).
\end{equation}

We can also invoke  Lemma \ref{lem:confiset-2} (b) and Lemma \ref{lem:de-regret} with
\begin{equation*}
\left\{
\begin{aligned}
& \rho =\frac{1}{K},\ \omega=  \sqrt{\frac{1}{K}},\ C=1, \\
&	\Xcal = \Scal\times\Acal, \ \Phi=(I-\Tcal_h)\Fcal, \ \mbox{and } \Pi=\Dcal_{\dirac,h},\\
& \phi_k = f^k_h-\Tcal_h f^k_{h+1}\mbox{ and } \mu_k = \one\{\cdot=(s_h^k,a_h^k)\},
\end{aligned}
\right.
\end{equation*}
and obtain 
\begin{equation}\label{eq:algx-step3}
\begin{aligned}
	\sum_{t=1}^{k} \Ecal(f^{t},\pi^{t},h) 
	\le &\sum_{t=1}^{k} (f^{t}_h-\Tcal f^t_{h+1})(s_h^t,a_h^t) + \Ocal\paren{\sqrt{k\log(k)}} \\
	\le &
	 \Ocal\left(\sqrt{k\cdot \text{\bedim}(\Fcal,\Dcal_{\dirac},\sqrt{1/K}) \log[KH\Ncal_{\Fcal\cup\Gcal}({1}/{K})/\delta]} \right),
\end{aligned}
\end{equation}
where the first inequality follows from standard martingale concentration.

Plugging either equation \eqref{eq:algx-step2} or \eqref{eq:algx-step3} back into equation \eqref{eq:algx-step1} completes the proof.

\subsection{Proof of Corollary \ref{cor:sample_golf}}
\label{appendix:algx-pac}

\paragraph{Step 1. Bounding the regret by Bellman error}
By Lemma \ref{lem:confiset-1}, we can upper bound the cumulative regret by the summation of Bellman error with probability at least $1-\delta$:
\begin{equation}\label{eq:algx-step1-9}
	\begin{aligned}
	\sum_{k=1}^K \paren{V_1^\star(s_1) - V^{\pi^k}_1(s_1)}
	\le \sum_{k=1}^K \paren{\max_{a}f^k_1(s_1,a) - V^{\pi^k}_1(s_1)}
	\overset{(i)}{=}
	\sum_{k=1}^{K} \sum_{h=1}^{H} \Ecal(f^{k},\pi^{k},h),
\end{aligned}
\end{equation}
where $(i)$ follows from standard policy loss decomposition  (e.g. Lemma 1 in \cite{jiang2017contextual}). 

\paragraph{Step 2. Bounding cumulative Bellman error using DE dimension}
Next, we focus on a fixed step $h$ and bound the cumulative Bellman error $\sum_{k=1}^{K} \Ecal(f^{k},\pi^{k},h)$ using Lemma \ref{lem:confiset-2}.
 
If we invoke  Lemma \ref{lem:confiset-2} (a) 
with 
$$
\rho =\frac{\epsilon^2}{H^2\cdot \text{\bedim}(\Fcal,\Dcal_{\Fcal},\epsilon/H)},
$$
and Lemma \ref{lem:de-regret} with
\begin{equation*}
\left\{
\begin{aligned}
&  \omega=  \frac{\epsilon}{H},\ C=1, \\
&	\Xcal = \Scal\times\Acal, \ \Phi=(I-\Tcal_h)\Fcal, \ \Pi=\Dcal_{\Fcal,h},\\
& \phi_k = f^k_h-\Tcal_h f^k_{h+1} \mbox{ and } \mu_k = \Pr^{\pi^{k}}(s_h=\cdot,a_h=\cdot),
\end{aligned}
\right.
\end{equation*}
we obtain with probability at least $1-10^{-3}$,
\begin{equation}\label{eq:algx-step2-9}
\begin{aligned}
	\frac{1}{K}\sum_{k=1}^{K} \Ecal(f^{k},\pi^k,h) \le &\Ocal\left(\sqrt{\text{\bedim}(\Fcal,\Dcal_{\Fcal},\epsilon/H)[\frac{  \log[KH\Ncal_{\Fcal\cup\Gcal}(\rho)]}{K}+\rho]} + 
	\frac{\epsilon}{H}
	\right)\\
	\le & \Ocal\left( \frac{\epsilon}{H} + \sqrt{\frac{ d \log[KH\Ncal_{\Fcal\cup\Gcal}(\rho)]}{K}}
	\right),
\end{aligned}
\end{equation}
where the second inequality follows from the choice of $\rho$ and $d:=\text{\bedim}(\Fcal,\Dcal_{\Fcal},\epsilon/H)$. Now we need to choose $K$ such that 
\begin{equation}
	\sqrt{\frac{ d \log[KH\Ncal_{\Fcal\cup\Gcal}(\rho)]}{K}} \le \frac{\epsilon}{H}.
\end{equation}
By simple calculation, one can verify it suffices to choose 
\begin{equation}
	K = \frac{H^2 d \log(Hd\Ncal_{\Fcal\cup\Gcal}(\rho)/\epsilon)}{\epsilon^2}.
\end{equation}

Plugging equation \eqref{eq:algx-step2-9}  back into equation \eqref{eq:algx-step1-9} completes the proof.
We can similarly prove the bound in terms of the \be\ dimension with respect to $\Dcal_{\dirac}$.

\subsection{Proofs of concentration lemmas}
\label{appendix:concentration}


To begin with, recall the Freedman's inequality that controls the sum of martingale difference by the sum of their predicted variance.
\begin{lemma}[Freedman's inequality {\citep[e.g.,][]{agarwal2014taming}}]\label{lem:freedman1}
Let $(Z_t)_{t \leq T}$ be a real-valued martingale difference sequence adapted to filtration $\mathfrak{F}_t$, and let $\E_t[\cdot]=\E[\cdot\  | \ \mathfrak{F}_t]$. If $|Z_t|\leq R$ almost surely, then for any $\eta \in (0,\frac{1}{R})$ it holds that with probability at least $1-\delta$,
$$
	\sum_{t=1}^{T}Z_t \leq \Ocal\paren{\eta \sum_{t=1}^{T}\E_{t-1}[Z_t^2]+\frac{\log(\delta^{-1})}{\eta}}.
$$
\end{lemma}

\subsubsection{Proof of Lemma \ref{lem:confiset-2}}
\label{concentration:template}
\begin{proof}
We prove inequality $(b)$ first.

	Consider a fixed $(k,h,f)$ tuple. 
Let $$X_t(h,f) := (f_h(s_h^t,a_h^t)- r_h^t - f_{h+1}(s_{h+1}^t,\pi_f(s_{h+1}^t)))^2 - ((\Tcal f_{h+1})(s_h^t,a_h^t)- r_h^t - f_{h+1}(s_{h+1}^t,\pi_f(s_{h+1}^t)))^2 $$ and $\Ffrak_{t,h}$ be the filtration induced by $\{s_1^i,a_1^i,r_1^i,\ldots,s_H^i\}_{i=1}^{t-1}\bigcup\{s_1^t,a_1^t,r_1^t,\ldots,s_h^t,a_h^t\}$. We have 
$$\E [X_t(h,f) \mid \Ffrak_{t,h}] = [(f_h-\Tcal f_{h+1})(s_h^t,a_h^t)]^2$$
and
$$\Var[X_t(h,f)\mid \Ffrak_{t,h}]\le \E [(X_t(h,f))^2 \mid \Ffrak_{t,h}] \le 36[(f_h-\Tcal f_{h+1})(s_h^t,a_h^t)]^2=
36\E [X_t(h,f) \mid \Ffrak_{t,h}]. $$ 
By Freedman's inequality, we have, with probability at least $1-\delta$, 
\begin{equation*}
\left| \sum_{t=1}^{k} X_t(h,f) - \sum_{t=1}^k\E [X_t(h,f) \mid \Ffrak_{t,h}] \right| \le 
\Ocal \paren{\sqrt{\log(1/\delta)\sum_{t=1}^k\E [X_t \mid \Ffrak_{t,h}]} + \log(1/\delta)}.
\end{equation*}
Let $\Zcal_{\rho}$ be a $\rho$-cover of $\Fcal$.
Now taking a union bound for all $(k,h,\phi)\in[K]\times[H]\times\Zcal_\rho$, we obtain that with probability at least $1-\delta$,  for  all $(k,h,\phi)\in[K]\times[H]\times\Zcal_\rho$
 \begin{equation}
 \label{eq:Nov27-1}
\left| \sum_{t=1}^{k} X_t(h,\phi) - \sum_{t=1}^k[(\phi_h-\Tcal \phi_{h+1})(s_h^t,a_h^t)]^2 \right| \le 
\Ocal \paren{\sqrt{\iota\sum_{t=1}^k[(\phi_h-\Tcal \phi_{h+1})(s_h^t,a_h^t)]^2} + \iota},
\end{equation}
where  $\iota=\log(HK|\Zcal_\rho|/\delta)$.
From now on, we will do all the analysis conditioning on this event being true. 

Consider an arbitrary $(h,k)\in[H]\times[K]$ pair. 
By the definition of $\Bcal^{k}$ and Assumption \ref{asp:G-completeness}
\begin{equation*}
\begin{aligned}
	\sum_{t=1}^{k-1} X_t(h,f^{k})
=&  \sum_{t=1}^{k-1} [f^{k}_h(s_h^t,a_h^t)- r_h^t - f^{k}_{h+1}(s_{h+1}^t,\pi_{f^{k}}(s_{h+1}^t))]^2\\
& - \sum_{t=1}^{k-1}[(\Tcal f^{k}_{h+1})(s_h^t,a_h^t)- r_h^t - f^{k}_{h+1}(s_{h+1}^t,\pi_{f^{k}}(s_{h+1}^t))]^2 \\
 \le &  \sum_{t=1}^{k-1} [f^{k}_h(s_h^t,a_h^t)- r_h^t - f^{k}_{h+1}(s_{h+1}^t,\pi_{f^{k}}(s_{h+1}^t))]^2\\
& - \inf_{g\in\Gcal}\sum_{t=1}^{k-1}[g_h(s_h^t,a_h^t)- r_h^t - f^k_{h+1}(s_{h+1}^t,\pi_{f^k}(s_{h+1}^t))]^2 \le \beta.
\end{aligned}
\end{equation*}

Define $\phi^k=\argmin_{\phi\in\Zcal_\rho}\max_{h\in[H]}\|f_h^k-\phi_h^k\|_\infty$.
By the definition of $\Zcal_\rho$, we have
\begin{equation*}
\left|		 \sum_{t=1}^{k-1} X_t(h,f^k) 
- \sum_{t=1}^{k-1} X_t(h,\phi^k)\right| \le \Ocal(k\rho).
\end{equation*}
Therefore,
\begin{equation}\label{eq:Jan25-1}
 \sum_{t=1}^{k-1} X_t(h,\phi^k) \le \Ocal(k\rho)+\beta.
\end{equation}
Recall inequality \eqref{eq:Nov27-1} implies
 \begin{equation}
 \label{eq:Jan25-2}
\left| \sum_{t=1}^{k-1} X_t(h,\phi^k) - \sum_{t=1}^{k-1}[(\phi^k_h-\Tcal \phi^k_{h+1})(s_h^t,a_h^t)]^2 \right| \le 
\Ocal \paren{\sqrt{\iota\sum_{t=1}^{k-1}[(\phi^k_h-\Tcal \phi^k_{h+1})(s_h^t,a_h^t)]^2} + \iota}.
\end{equation}
Putting \eqref{eq:Jan25-1} and \eqref{eq:Jan25-2} together, we obtain
\begin{equation*}
\sum_{t=1}^{k-1} [(\phi^k_h-\Tcal \phi^k_{h+1})(s_h^t,a_h^t)]^2   \le \Ocal(\iota+k\rho + \beta).	
\end{equation*}
Because $\phi^k$ is an $\rho$-approximation to $f^k$, we conclude
\begin{equation*}
\sum_{t=1}^{k-1} [(f^k_h-\Tcal f^k_{h+1})(s_h^t,a_h^t)]^2   \le \Ocal(\iota+k\rho+ \beta).	
\end{equation*}
Therefore, we prove  inequality $(b)$ in Lemma \ref{lem:confiset-2}.

To prove inequality $(a)$, we only need to redefine $\Ffrak_{t,h}$ to be the filtration induced by \\
$\{s_1^i,a_1^i,r_1^i,\ldots,s_H^i\}_{i=1}^{t-1}$
and then repeat the arguments above verbatim.
\end{proof}

\subsubsection{Proof of Lemma \ref{lem:confiset-1}}
\label{concentration:template2}

\begin{proof}
Let $\Vcal_\rho$ be a $\rho$-cover of $\Gcal$.

Consider an arbitrary fixed tuple $(k,h,g)\in[K]\times[H]\times\Gcal$. 
Let $$W_t(h,g) := (g_h(s_h^t,a_h^t)- r_h^t - Q_{h+1}^\star(s_{h+1}^t,\pi_{Q^\star}(s_{h+1}^t)))^2 - (Q_h^\star(s_h^t,a_h^t)- r_h^t - Q_{h+1}^\star(s_{h+1}^t,\pi_{Q^\star}(s_{h+1}^t)))^2 $$ and $\Ffrak_{t,h}$ be the filtration induced by $\{s_1^i,a_1^i,r_1^i,\ldots,s_H^i\}_{i=1}^{t-1}\bigcup\{s_1^t,a_1^t,r_1^t,\ldots,s_h^t,a_h^t\}$. We have 
$$\E [W_t(h,g) \mid \Ffrak_{t,h}] = [(g_h- Q^\star_{h})(s_h^t,a_h^t)]^2$$
and
$$\Var[W_t(h,g)\mid \Ffrak_{t,h}]\le \E [(W_t(h,g))^2 \mid \Ffrak_{t,h}] \le 36((g_h-Q_h^\star)(s_h^t,a_h^t))^2=
36\E [W_t(h,g) \mid \Ffrak_{t,h}]. $$ 
By Freedman's inequality, with probability at least $1-\delta$, 
\begin{equation*}
\left| \sum_{t=1}^{k} W_t(h,g) - \sum_{t=1}^k[(g_h- Q_h^\star)(s_h^t,a_h^t)]^2 \right| \le 
\Ocal \paren{\sqrt{\log(1/\delta)\sum_{t=1}^k[(g_h- Q_h^\star)(s_h^t,a_h^t)]^2} + \log(1/\delta)}.
\end{equation*}
By taking a union bound  
over $[K]\times[H]\times\Vcal_\rho$ and the non-negativity of $\sum_{t=1}^k[(g_h- Q_h^\star)(s_h^t,a_h^t)]^2$, we obtain that with probability at least $1-\delta$, for  all $(k,h,\psi)\in[K]\times[H]\times\Vcal_\rho$
$$
-\sum_{t=1}^{k} W_t(h,\psi) \le \Ocal(\iota),
$$
where  $\iota=\log(HK|\Vcal_\rho|/\delta)$.
This directly implies for  all $(k,h,g)\in[K]\times[H]\times\Gcal$
\begin{align*}
&\sum_{t=1}^{k-1} [Q_h^\star(s_h^t,a_h^t)- r_h^t - Q_{h+1}^\star(s_{h+1}^t,\pi_{Q^\star}(s_{h+1}^t))]^2 \\
\le &\sum_{t=1}^{k-1}[g_h(s_h^t,a_h^t)- r_h^t - Q^\star_{h+1}(s_{h+1}^t,\pi_{Q^\star}(s_{h+1}^t))]^2 
+\Ocal(\iota+k\rho).	
\end{align*}
Finally, by recalling the definition of $\Bcal^k$, we conclude 
that with probability at least $1-\delta$, $Q^\star\in\Bcal^{k}$ for all $k\in[K]$.
\end{proof}

%
%

%
%
%
%
%
%

\subsection{Proof of Lemma \ref{lem:de-regret}}
\label{appendix:de-regret}

The proof in this subsection basically follows the same arguments as in  Appendix C of \cite{russo2013eluder}.
We firstly prove the following proposition which bounds the number of times $|\E_{\mu_t} [\phi_t]|$ can exceed a certain threshold.  

\begin{proposition}
	Given a function class $\Phi$ defined on $\Xcal$, and a family of probability measures $\Pi$ over $\Xcal$. 
	Suppose sequence $\{\phi_k\}_{k=1}^{K}\subset \Phi$ and $\{\mu_k\}_{k=1}^{K}\subset\Pi$ satisfy that for all $k\in[K]$,
	$\sum_{t=1}^{k-1} (\E_{\mu_t} [\phi_k])^2 \le \beta$. Then for all $k\in[K]$,
$$
	\sum_{t=1}^{k} \one\big \{|\E_{\mu_t} [\phi_t]| > \epsilon \big \} \leq (\frac{\beta}{\epsilon^2}+1)\dedim (\Phi,\Pi,\epsilon).
$$
\label{prop:de-regret-prop}
\end{proposition}
\begin{proof}[Proof of Proposition \ref{prop:de-regret-prop}]
	 We first show that if for some $k$ we have $ |\E_{\mu_k} [\phi_k]| > \epsilon $, then $\mu_k$ is $\epsilon$-dependent on at most $\beta / \epsilon^2$ disjoint subsequences in $\{\mu_1,\dots,\mu_{k-1}\}$. By definition of DE dimension, if $|\E_{\mu_k} [\phi_k]| > \epsilon$ and $\mu_k$ is $\epsilon$-dependent on a subsequence $\{\nu_{1},\dots,\nu_{\ell}\}$ of $\{\mu_1,\dots,\mu_{k-1}\}$, then we should have $\sum_{t=1}^{\ell} (\E_{\nu_{t}} [\phi_k])^2 \geq \epsilon^2$. It implies that if $\mu_k$ is $\epsilon$-dependent on $L$ disjoint subsequences in $\{\mu_1,\dots,\mu_{k-1}\}$, we have 
	 $$  \beta \geq \sum_{t=1}^{k-1} (\E_{\mu_{t}} [\phi_k])^2 \geq L\epsilon^2$$ resulting in $L \leq {\beta}/{\epsilon^2}$.
	 
	 Now we want to show that for any sequence $\{\nu_1,\dots,\nu_\kappa\} \subseteq \Pi$, there exists $ j \in [\kappa]$ such that $\nu_j$ is $\epsilon$-dependent on at least $L = \lceil (\kappa-1) / \dedim (\Phi,\Pi,\epsilon)\rceil $ disjoint subsequences in  $\{\nu_1,\dots,\nu_{j-1}\}$. 
	 We argue by the following mental procedure: we start with singleton sequences $B_1 = \{\nu_1\}, \dots,B_L$ $= \{\nu_L\}$ and $j=L+1$. For each $j$, if $\nu_j$ is $\epsilon$-dependent on $B_1,\dots,B_L$ we already achieved our goal so we stop; otherwise, we pick an $i\in[L]$ such that $\nu_j$ is $\epsilon$-independent of $B_i$ and update $B_i=B_i\cup\{\nu_j\}$. Then we increment $j$ by $1$ and continue this process. By the definition of DE dimension, the size of each $B_1,\dots,B_L$ cannot get bigger than $\dedim (\Phi,\Pi,\epsilon)$ at any point in this process. Therefore, the process stops before or on $j=L \dedim (\Phi,\Pi,\epsilon)+1 \leq \kappa$.
	 
	 Fix $k \in [K]$ and let $\{\nu_1,\dots,\nu_\kappa\}$ be subsequence of $\{\mu_{1},\dots,\mu_{k}\}$, consisting of elements for which $|\E_{\mu_t} [\phi_t]| > \epsilon$. Using the first claim, we know that each $\nu_j$ is $\epsilon$-dependent on at most $\beta / \epsilon^2$ disjoint subsequences of $\{\nu_1,\dots,\nu_{j-1}\}$. Using the second claim, we know there exists $j \in [\kappa]$ such that $\nu_j$ is $\epsilon$-dependent on at least $(\kappa / \dedim (\Phi,\Pi,\epsilon)) - 1$ disjoint subsequences of $\{\nu_1,\dots,\nu_{j-1}\}$. Therefore, we have $\kappa/ \dedim (\Phi,\Pi,\epsilon)-1 \leq \beta / \epsilon^2$ which results in 
	 $$
	 	\kappa \leq (\frac{\beta}{\epsilon^2}+1)\dedim (\Phi,\Pi,\epsilon)
	 $$
	 and completes the proof.
 \end{proof}
 
\begin{proof}[Proof of Lemma \ref{lem:de-regret}]
	Fix $k \in [K]$; let $d = \dedim (\Phi,\Pi,\omega)$. Sort the sequence $\{|\E_{\phi_1} [\phi_1]|,\dots,$\\
	$|\E_{\mu_k} [\phi_k]|\}$ in a decreasing order and denote it by $\{e_1,\dots,e_k\}$ ($e_1 \geq e_2 \geq \dots \geq e_k$).
	$$
		\sum_{t=1}^k |\E_{\mu_t} [\phi_t]| = \sum_{t=1}^k e_t = \sum_{t=1}^k e_t \one\big \{e_t \leq \omega \big \} + \sum_{t=1}^k e_t \one\big \{e_t > \omega\big\} \leq k\omega + \sum_{t=1}^k e_t \one\big \{e_t > \omega\big\}.
	$$
	For $t \in [k]$, we want to prove that if $e_t > \omega$, then we have $e_t \leq \min\{\sqrt{\frac{d\beta}{t-d}},C\}$. Assume $t\in[k]$ satisfies $e_t >  \omega$. 
	Then there exists $\alpha$ such that $e_t >\alpha\ge  \omega$.
	By Proposition \ref{prop:de-regret-prop}, we have
	$$
	t \le 
	\sum_{i=1}^k \one\big \{e_i > \alpha\big\} 
	\le \big( \frac{\beta}{\alpha^2} + 1 \big) \dedim (\Phi,\Pi,\alpha)
	\le 
	\big( \frac{\beta}{\alpha^2} + 1 \big) \dedim (\Phi,\Pi,\omega),
	$$
	which implies $\alpha\le \sqrt{\frac{d\beta}{t-d}}$.
	Besides, recall $e_t \leq C$, so we have   $e_t \leq \min\{\sqrt{\frac{d\beta}{t-d}},C\}$.
	
Finally, we have
	\begin{equation*}
	\begin{aligned}
			\sum_{t=1}^k e_t \one\big \{e_t > \omega\big\} \leq \min\{d,k\}C+\sum_{t=d+1}^k \sqrt{\frac{d\beta}{t-d}} &\leq \min\{d,k\}C+ \sqrt{d\beta}\int_{0}^k \frac{1}{\sqrt{t}}dt \\
			&\leq \min\{d,k\}C + 2\sqrt{d\beta k},
	\end{aligned}
	\end{equation*}
	which completes the proof.
\end{proof}

\section{Proofs for \olive}
\label{appendix:olive-proof}

In this section, we provide the formal proof for the results stated in Appendix \ref{sec:olive}.

\subsection{Full proof of Theorem \ref{thm:olive}}

\begin{proof}[Proof of Theorem \ref{thm:olive}]
By standard concentration arguments (Hoeffding's inequality plus union bound argument),
 with probability at least $1-\delta$,
 the following events hold for the first $dH+1$ phases (please refer to Appendix \ref{olive:concentration} for the proof)
\begin{enumerate}
	\item  If the elimination procedure is activated at the $h^{\rm th}$ step in the $k^{\rm th}$ phase, then  $\Ecal(f^k,\pi^{k},h)>\zeta_{\rm act}/2$ and all $f\in\Fcal$ satisfying $|\Ecal(f,\pi^{k},h)|\ge2\zeta_{\rm elim}$ get eliminated.
	\item  If the elimination procedure is not activated in the $k^{\rm th}$ phase, then, $\sum_{h=1}^{H}\Ecal(f^k,\pi^{k},h)<2H\zeta_{\rm act}=4\epsilon$.
	\item $Q^\star$ is not eliminated. 
\end{enumerate}
Therefore, if we can show \olive\ terminates within $d H+1$ phases, then with high probability the output policy is  $4\epsilon$-optimal by the optimism of $f^k$ and simple policy loss decomposition  (e.g. Lemma 1 in \cite{jiang2017contextual}):
\begin{equation}
	\begin{aligned}
	 \paren{V_1^\star(s_1) - V^{\pi^{k}}_1(s_1)}
	\le  \max_{a}f^k(s_1,a) - V^{\pi^{k}}(s_1)
	{=}
	 \sum_{h=1}^{H} \Ecal(f^{k},\pi^k,h)
	 \le 4\epsilon.
\end{aligned}
\end{equation}

In order to prove that \olive\ terminates within $d H+1$ phases, it suffices to show that for each $h\in[H]$, we can activate the elimination procedure at the $h^{\rm th}$ step for at most  $d$ times.

For the sake of contradiction, assume that \olive\ does not terminate in $dH+1$ phases. Within these $dH+1$ phases, there exists some $h \in [H]$ for which the activation process has been activated for at least $d+1$ times. Denote by $k_1<\cdots<k_{d+1} \le dH+1 $ the indices of the phases where the elimination is activated at  the $h^{\rm th}$  step. By the high-probability events, for all $i<j\le d+1$, we have $|\Ecal(f^{k_j},\pi^{k_i},h)|<2\zeta_{\rm elim}$ and for all $l\le d+1$, we have $\Ecal(f^{k_l},\pi^{{k_l}},h)>\zeta_{\rm act}/2$. This means for all $l \le d+1$, we have both
$\sqrt{\sum_{i=1}^{l-1} \big(\Ecal(f^{k_l},\pi^{{k_{i}}},h)\big)^2}< \sqrt{d}\times 2\zeta_{\rm elim} = \epsilon/H$ and 
$\Ecal(f^{k_l},\pi^{{k_l}},h)>\zeta_{\rm act}/2 =\epsilon/H$.
	Therefore, the roll-in distribution of $\pi^{k_1},\ldots,\pi^{k_{d+1}}$ at step $h$ is an $\epsilon/H$-independent sequence of length $d+1$, which contradicts with the definition of BE  dimension.
	So \olive\ should terminate within $d H+1$ phases.
	
In sum, with probability at least $1-\delta$, 
Algorithm \ref{alg:olive} will terminate and output a $4\epsilon$-optimal policy using at most 
$$
(dH+1)(\nact +\nelim) \le
\frac{3cH^3d^2\log(\Ncal(\Fcal,\zeta_{\rm elim}/8))\cdot \iota}{\epsilon^2}$$
 episodes.
	\end{proof}

\subsection{Concentration arguments for Theorem \ref{thm:olive}}	
\label{olive:concentration}
Recall in Algorithm \ref{alg:olive} we 
choose $$\zeta_{\rm act}=\frac{2\epsilon}{H},\ \zetaelim=\frac{\epsilon}{2H\sqrt{d}},\ \nact=\frac{c H^2\iota}{\epsilon^2},\text{ and } \nelim=\frac{cH^2d\log(\Ncal(\Fcal,\zeta_{\rm elim}/8))\cdot \iota}{\epsilon^2},$$
 where $d=\max_{h\in[H]}\text{\bedim}\big(\Fcal,\Dcal_{\Fcal,h},\epsilon/H\big)$, $\iota=\log[Hd/\delta\epsilon]$ and $c$ is a large absolute constant.Our goal is to prove  with probability at least $1-\delta$,
 the following events hold for the first $dH+1$ phases
\begin{enumerate}
	\item  If the elimination procedure is activated at the $h^{\rm th}$ step in the $k^{\rm th}$ phase, then  $\Ecal(f^k,\pi^{k},h)>\zeta_{\rm act}/2$ and all $f\in\Fcal$ satisfying $|\Ecal(f,\pi^{k},h)|\ge2\zeta_{\rm elim}$ get eliminated.
	\item  If the elimination procedure is not activated in the $k^{\rm th}$ phase, then, $\sum_{h=1}^{H}\Ecal(f^k,\pi^{k},h)<2H\zeta_{\rm act}=4\epsilon$.
	\item $Q^\star$ is not eliminated. 
\end{enumerate}

We begin with the activation procedure.

\paragraph{Concentration in the activation procedure}
Consider a fixed $(k,h)\in[dH+1]\times[H]$ pair.
By Azuma-Hoefdding's inequality, with probability at least $1-\frac{\delta}{8H(dH^2+1)}$, we have 
$$
| \hatberr{f^k, \pi^{k}, h}  - \berr{f^k, \pi^{k}, h}| \le \Ocal\paren{\sqrt{\frac{\iota}{n_{\rm act}}}} \le \frac{\epsilon}{2H}\le \zeta_{\rm act}/4,
$$
where the second inequality follows from $\nact=C\frac{H^2\iota}{\epsilon^2}$ with $C$ being chosen large enough.

Take a union bound for all $(k,h)\in[dH+1]\times[H]$, we have with probability at least $1-{\delta}/4$, the following holds for all $(k,h)\in[dH+1]\times[H]$
$$
| \hatberr{f^k, \pi^{k}, h}  - \berr{f^k, \pi^{k}, h}|\le\zeta_{\rm act}/4.
$$
By Algorithm \ref{alg:olive}, if the elimination procedure is not activated in the  $k^{\rm th}$ phase, we have\\
 $\sum_{h=1}^{H}\hat{\Ecal}(f^k,\pi^{k},h)\le H\zeta_{\rm act}$. Combine it with the concentration argument we just proved,
$$
\sum_{h=1}^{H}{\Ecal}(f^k,\pi^{k},h)\le 
\sum_{h=1}^{H}\hat{\Ecal}(f^k,\pi^{k},h) + \frac{H\zeta_{\rm act}}{4} <\frac{5H\zeta_{\rm act}}{4}.
$$
On the other hand, if the elimination procedure is activated at the $h^{\rm th}$ step in the $k^{\rm th}$ phase, then  $\hat{\Ecal}(f^k,\pi^{k},h)>\zeta_{\rm act}$. Again combine it with the concentration argument we just proved,
$$
{\Ecal}(f^k,\pi^{k},h)\ge \hat{\Ecal}(f^k,\pi^{k},h)-\frac{\zeta_{\rm act}}{4} > \frac{3\zeta_{\rm act}}{4}.
$$

\paragraph{Concentration in the elimination procedure}
Now, let us turn to the elimination procedure. 
First, let $\Zcal$ be an ${\zetaelim}/{8}$-cover of $\Fcal$ with cardinality 
$\Ncal(\Fcal,\zeta_{\rm elim}/8)$. 
With a little abuse of notation, for every $f\in\Fcal$, define $\fhat=\argmin_{g\in\Zcal}\max_{h\in[H]}\|f_h-g_h\|_\infty$.
By applying Azuma-Hoeffding's inequality to all $(k,g)\in[dH+1]\times\Zcal$ and taking a union bound, we have with probability at least $1-{\delta}/4$, the following holds for all $(k,g)\in[dH+1]\times\Zcal$
$$
| \hatberr{g, \pi^{k}, h_k}  - \berr{g, \pi^{k}, h_k}| \le \zeta_{\rm elim}/4.
$$

Recall that Algorithm \ref{alg:olive} eliminates all $f$ satisfying $|\hatberr{f, \pi^{k}, h_k}| > \zeta_{\rm elim}$ when the elimination procedure is activated at the $h^{\rm th}_{k}$ step in the $k^{\rm th}$ phase. 
Therefore, if 
$|\Ecal(f,\pi^{k},h_k)|\ge2\zeta_{\rm elim}$, 
$f$ will be eliminated because 
\begin{align*}
	|\hat{\Ecal}(f,\pi^{k},h_k)|
&\ge |\hat{\Ecal}(\fhat,\pi^{k},h_k)|- 2\times \frac{\zetaelim}{8}\\
&\ge  | \berr{\fhat, \pi^{k}, h_k}|  -\frac{\zeta_{\rm elim}}{2}\\
&\ge | \berr{f, \pi^{k}, h_k}|  -\frac{\zeta_{\rm elim}}{2}- 2\times \frac{\zetaelim}{8}
> \zeta_{\rm elim}.
\end{align*}
Finally, note that $\Ecal(Q^\star,\pi,h)\equiv 0$ for any $\pi$ and $h$. As a result, it will never be eliminated within the first $dH+1$ phases because we can similarly prove 
$$|\hat{\Ecal}(Q^\star,\pi^{k},h_k)| \le  | \berr{Q^\star, \pi^{k}, h_k}|  +\frac{3\zeta_{\rm elim}}{4} < \zeta_{\rm elim}.$$

{\bf Wrapping up}: take a union bound for the activation and elimination procedure, and conclude that the three events, listed at the beginning of this section, hold for the the first $dH+1$ phases
 with probability at least $1-\delta/2$.

\section{Proofs for V-type Variants}

In this section, we provide formal proofs for the results stated in Section \ref{app:BE-typeII}.

\subsection{Proof of Theorem~\ref{thm:olive-typeII}}
\label{app:proof:olive-typeII}

The proof is similar to  that in Appendix~\ref{appendix:olive-proof}.
\begin{proof}[Proof of Theorem \ref{thm:olive-typeII}]
By standard concentration arguments (Hoeffding's inequality, Bernstein's inequality, and union bound argument),
 with probability at least $1-\delta$,
 the following events hold for the first $dH+1$ phases (please refer to Appendix \ref{olive:concentration-typeII} for the proof)
\begin{enumerate}
	\item  If the elimination procedure is activated at the $h^{\rm th}$ step in the $k^{\rm th}$ phase, then  $\EcalII(f^k,\pi^{k},h)>\zeta_{\rm act}/2$ and all $f\in\Fcal$ satisfying $|\EcalII(f,\pi^{k},h)|\ge2\zeta_{\rm elim}$ get eliminated.
	\item  If the elimination procedure is not activated in the $k^{\rm th}$ phase, then, $\sum_{h=1}^{H}\EcalII(f^k,\pi^{k},h)<2H\zeta_{\rm act}=4\epsilon$.
	\item $Q^\star$ is not eliminated. 
\end{enumerate}
Therefore, if we can show \olive\ terminates within $d H+1$ phases, then with high probability the output policy is  $4\epsilon$-optimal by the optimism of $f^k$ and simple policy loss decomposition  (e.g., Lemma 1 in \cite{jiang2017contextual}):
\begin{equation}
	\begin{aligned}
	 \paren{V_1^\star(s_1) - V^{\pi^{k}}_1(s_1)}
	\le  \max_{a}f^k(s_1,a) - V^{\pi^{k}}(s_1)
	{=}
	 \sum_{h=1}^{H} \EcalII(f^{k},\pi^k,h)
	 \le 4\epsilon.
\end{aligned}
\end{equation}

In order to prove that \olive\ terminates within $d H+1$ phases, it suffices to show that for each $h\in[H]$, we can activate the elimination procedure at the $h^{\rm th}$ step for at most  $d$ times.

For the sake of contradiction, assume that \olive\ does not terminate in $dH+1$ phases. Within these $dH+1$ phases, there exists some $h \in [H]$ for which the activation process has been activated for at least $d+1$ times. Denote by $k_1<\cdots<k_{d+1} \le dH+1 $ the indices of the phases where the elimination is activated at  the $h^{\rm th}$  step. By the high-probability events, for all $i<j\le d+1$, we have $|\EcalII(f^{k_j},\pi^{k_i},h)|<2\zeta_{\rm elim}$ and for all $l\le d+1$, we have $\EcalII(f^{k_l},\pi^{{k_l}},h)>\zeta_{\rm act}/2$. This means for all $l \le d+1$, we have both
$\sqrt{\sum_{i=1}^{l-1} \big(\EcalII(f^{k_l},\pi^{{k_{i}}},h)\big)^2}< \sqrt{d}\times 2\zeta_{\rm elim} = \epsilon/H$ and 
$\EcalII(f^{k_l},\pi^{{k_l}},h)>\zeta_{\rm act}/2 =\epsilon/H$.
		Therefore, the roll-in distribution of $\pi^{k_1},\ldots,\pi^{k_{d+1}}$ at step $h$ is an $\epsilon/H$-independent sequence of length $d+1$ with respect to $(I-\Tcal_h)V_{\Fcal}$, which contradicts with the definition of BE  dimension.
	So \olive\ should terminate within $dH+1$ phases.

%
	
In sum, with probability at least $1-\delta$, 
Algorithm \ref{alg:olive} will terminate and output a $4\epsilon$-optimal policy using at most 
$$
(dH+1)(\nact +\nelim) \le
\frac{3cH^3d^2|\Acal|\log(|\Fcal|)\cdot \iota}{\epsilon^2}$$
 episodes.
	\end{proof}

\newcommand{\berrII}[1]{\EcalII(#1)}
\subsubsection{Concentration arguments for Theorem \ref{thm:olive-typeII}}	
\label{olive:concentration-typeII}
Recall in Algorithm \ref{alg:olive-typeII} we 
choose $$\zeta_{\rm act}=\frac{2\epsilon}{H},\ \zetaelim=\frac{\epsilon}{2H\sqrt{d}},\ \nact=\frac{c H^2\iota}{\epsilon^2},\text{ and } \nelim=\frac{c|\Acal|H^2d\log(\Ncal(\Fcal,\zeta_{\rm elim}/8))\cdot \iota}{\epsilon^2},$$
 where $d=\max_{h\in[H]}\text{\bedimII}\big(\Fcal,\Dcal_{\Fcal,h},\epsilon/H\big)$, $\iota=\log[Hd/\delta\epsilon]$ and $c$ is a large absolute constant. Our goal is to prove  with probability at least $1-\delta$,
 the following events hold for the first $dH+1$ phases
\begin{enumerate}
	\item  If the elimination procedure is activated at the $h^{\rm th}$ step in the $k^{\rm th}$ phase, then  $\EcalII(f^k,\pi^{k},h)>\zeta_{\rm act}/2$ and all $f\in\Fcal$ satisfying $|\EcalII(f,\pi^{k},h)|\ge2\zeta_{\rm elim}$ get eliminated.
	\item  If the elimination procedure is not activated in the $k^{\rm th}$ phase, then, $\sum_{h=1}^{H}\EcalII(f^k,\pi^{k},h) < 2H\zeta_{\rm act}=4\epsilon$.
	\item $Q^\star$ is not eliminated. 
\end{enumerate}

We begin with the activation procedure.

\paragraph{Concentration in the activation procedure}
Consider a fixed $(k,h)\in[dH+1]\times[H]$ pair.
By Azuma-Hoefdding's inequality, with probability at least $1-\frac{\delta}{8H(dH+1)}$, we have 
$$
| \tildeberrII{f^k, \pi^{k}, h}  - \berrII{f^k, \pi^{k}, h}| \le \Ocal\paren{\sqrt{\frac{\iota}{n_{\rm act}}}} \le \frac{\epsilon}{2H} \le \zeta_{\rm act}/4,
$$
where the second inequality follows from $\nact=C\frac{H^2\iota}{\epsilon^2}$ with $C$ being chosen large enough.

Take a union bound for all $(k,h)\in[dH+1]\times[H]$, we have with probability at least $1-{\delta}/4$, the following holds for all $(k,h)\in[dH+1]\times[H]$
$$
| \tildeberrII{f^k, \pi^{k}, h}  - \berrII{f^k, \pi^{k}, h}| \le \zeta_{\rm act}/4.
$$
By Algorithm \ref{alg:olive-typeII}, if the elimination procedure is not activated in the  $k^{\rm th}$ phase, we have\\
 $\sum_{h=1}^{H}\tildeberrII{f^k,\pi^{k},h}\le H\zeta_{\rm act}$. Combine it with the concentration argument we just proved,
$$
\sum_{h=1}^{H}{\EcalII}(f^k,\pi^{k},h)\le 
\sum_{h=1}^{H}\tildeberrII{f^k,\pi^{k},h} + \frac{H\zeta_{\rm act}}{4} \le \frac{5H\zeta_{\rm act}}{4}.
$$
On the other hand, if the elimination procedure is activated at the $h^{\rm th}$ step in the $k^{\rm th}$ phase, then  $\tildeberrII{f^k,\pi^{k},h}>\zeta_{\rm act}$. Again combine it with the concentration argument we just proved,
$$
{\EcalII}(f^k,\pi^{k},h)\ge \tildeberrII{f^k,\pi^{k},h}-\frac{\zeta_{\rm act}}{4} > \frac{3\zeta_{\rm act}}{4}.
$$

\paragraph{Concentration in the elimination procedure}
Now, let us turn to the elimination procedure. 
We start by bounding the the second moment of
$$
	\frac{\one[\pi_f(s_h)=a_h]}{1/|\Acal|} \big(f_h(s_h,a_h)-r_h-\max_{a' \in \Acal} f_{h+1}(s_{h+1},a')\big)
$$ for all $f \in \Fcal$. Let $y(s_h,a_h,r_h,s_{h+1})=f_h(s_h,a_h)-r_h-\max_{a' \in \Acal} f_{h+1}(s_{h+1},a') \in [-2,1]$, then we have
\begin{equation*}
\begin{aligned}
	&\E[\big(|\Acal|\one[\pi_f(s_h)=a_h]y(s_h,a_h,r_h,s_{h+1})\big)^2\mid s_h\sim \pi^k, a_h\sim\mathrm{Uniform}(\Acal)]\\
	\le & 4|\Acal|^2 \E[\one[\pi_f(s_h)=a_h]\mid s_h\sim \pi^k, a_h\sim\mathrm{Uniform}(\Acal)] = 4|\Acal|.
\end{aligned}
\end{equation*}
For a fixed $(k,f)\in[dH+1]\times\Fcal$, by applying Azuma-Bernstein's inequality, with probability at least $1-\frac{\delta}{8(dH+1)|\Fcal|}$ we have
\begin{equation*}
\begin{aligned}
	| \hatberrII{f, \pi^{k}, h_k}  - \berrII{f, \pi^{k}, h_k}| &\le \Ocal\left(\sqrt{\frac{|\Acal|\iota'}{\nelim}}+\frac{|\Acal| \iota'}{\nelim}\right)
	\le \Ocal\left(\sqrt{\frac{|\Acal|\iota'}{\nelim}}\right)
	\le \zeta_{\rm elim}/2,
\end{aligned}
\end{equation*}
where $\iota' = \log[8(dH+1)|\Fcal|/{\delta}]$, and the third inequality follows from $\nelim=C |\Acal|\iota / \zeta_{\rm elim}^2$ with $C$ being chosen large enough.

Taking a union bound over $[dH+1]\times\Fcal$, we have with probability at least $1-{\delta}/4$, the following holds for all $(k,f)\in[dH+1]\times\Fcal$
$$
| \hatberrII{f, \pi^{k}, h_k}  - \berrII{f, \pi^{k}, h_k}| \le \zeta_{\rm elim}/2.
$$

Recall that Algorithm \ref{alg:olive-typeII} eliminates all $f$ satisfying $|\hatberrII{f, \pi^{k}, h_k}| > \zeta_{\rm elim}$ when the elimination procedure is activated at the $h^{\rm th}_{k}$ step in the $k^{\rm th}$ phase. 
Therefore, if 
$|\EcalII(f,\pi^{k},h_k)|\ge2\zeta_{\rm elim}$, 
$f$ will be eliminated because 
\begin{align*}
	|\hatberrII{f,\pi^{k},h_k}|
&\ge | \berrII{f, \pi^{k}, h_k}|  -\frac{\zeta_{\rm elim}}{2} > \zeta_{\rm elim}.
\end{align*}
Finally, note that $\EcalII(Q^\star,\pi,h)\equiv 0$ for any $\pi$ and $h$. As a result, it will never be eliminated within the first $dH+1$ phases because we can similarly prove 
$$|\hatberrII{Q^\star,\pi^{k},h_k}| \le  | \berrII{Q^\star, \pi^{k}, h_k}|  +\frac{\zeta_{\rm elim}}{2} < \zeta_{\rm elim}.$$

{\bf Wrapping up}: take a union bound for the activation and elimination procedure, and conclude that the three events, listed at the beginning of this section, hold for the the first $dH+1$ phases
 with probability at least $1-\delta/2$.

\subsection{Proof of Theorem~\ref{thm:main-typeII}}
\label{app:proof:algx-typeII}

The proof is basically the same as 
that of Theorem \ref{thm:main} in 
 Appendix~\ref{appendix:lcgo}.

To begin with, we have the following lemma (akin to Lemma~\ref{lem:confiset-2} and \ref{lem:confiset-1}) showing that with high probability: $(i)$ any function in the confidence set has low Bellman-error over the collected Datasets $\Dcal_1,\dots,\Dcal_H$ as well as the distributions from which $\Dcal_1,\dots,\Dcal_H$ are sampled; $(ii)$ the optimal value function is inside the confidence set. Its proof is almost identical to that  of Lemma~\ref{lem:confiset-2} and \ref{lem:confiset-1} which can be found in Appendix~\ref{appendix:concentration}.

\begin{lemma}[Akin to Lemma~\ref{lem:confiset-2} and \ref{lem:confiset-1}]
\label{lem:confiset-typeII}
Let $\rho>0$ be an arbitrary fixed number.
If we choose $\beta=c \big( \log[KH\Ncal_{\Fcal\cup\Gcal}(\rho)/\delta]+K\rho\big)$ with some large absolute constant $c$ in Algorithm \ref{alg:algxII}, then
	with probability at least $1-\delta$, for all $(k,h)\in[K]\times[H]$, we have 
	\begin{enumerate}[label=(\alph*)]
		\item $\sum_{i=1}^{k-1}  \E [ \paren{f_h^k(s_h,a_h) - (\cT f_{h+1}^k)(s_h,a_h) }^2\mid s_h \sim \pi^i, a_h \sim {\rm Uniform}(\Acal)]
		{\le} \mathcal{O} (\beta)$,
		\item $\frac{1}{|\Acal|}\sum_{i=1}^{k-1}  \sum_{a\in\Acal}\paren{f^k_h(s_h^{i},a) - (\cT f_{h+1}^k)(s_h^{i},a) }^2
		{\le} \mathcal{O} ( \beta)$,
		\item $Q^\star\in \Bcal^k$,
	\end{enumerate}
	where $s_h^i$ denotes the state at step $h$ collected according to Line \ref{alg:algxII-diff1} in Algorithm \ref{alg:algxII} following $\pi^i$. 
\end{lemma}
\begin{proof}[Proof of Lemma \ref{lem:confiset-typeII}]
To prove inequality $(a)$, we only need to redefine the filtration $\Ffrak_{t,h}$ in Appendix \ref{concentration:template} to 
	 be the filtration induced by $\{s_1^i,a_1^i,r_1^i,\ldots,s_H^i\}_{i=1}^{t-1}$ and repeat the arguments there verbatim. 
	 
To prove inequality $(b)$, we only need to redefine the filtration $\Ffrak_{t,h}$ in Appendix \ref{concentration:template} to 
	 be the filtration induced by $\{s_1^i,a_1^i,r_1^i,\ldots,s_H^i\}_{i=1}^{t-1}\bigcup\{s_1^t,a_1^t,r_1^t,\ldots,s_h^t\}$ and repeat the arguments there verbatim.  	 
	 
The proof of $(c)$ is the same as that of Lemma \ref{lem:confiset-1} in Appendix \ref{concentration:template2}.
\end{proof}

\newcommand{\EcalU}{\Ecal_{U}}

\paragraph{Step 1. Bounding the regret by Bellman error}
By Lemma \ref{lem:confiset-typeII} $(c)$, we can upper bound the cumulative regret by the summation of Bellman error with probability at least $1-\delta$:
\begin{equation}\label{eq:proof-typeII-stepone}
	\begin{aligned}
	\sum_{k=1}^K \paren{V_1^\star(s_1) - V^{\pi^k}_1(s_1)}
	\le \sum_{k=1}^K \paren{\max_{a}f^k_1(s_1,a) - V^{\pi^k}_1(s_1)}
	\overset{(i)}{=}
	\sum_{k=1}^{K} \sum_{h=1}^{H} \EcalII(f^{k},\pi^{k},h),
\end{aligned}
\end{equation}
where $(i)$ follows from standard policy loss decomposition  (e.g. Lemma 1 in \cite{jiang2017contextual}). 

\paragraph{Step 2. Bounding cumulative Bellman error using \deber \ dimension}
Next, we focus on a fixed step $h$ and bound the cumulative Bellman error $\sum_{k=1}^{K} \EcalII(f^{k},\pi^{k},h)$ using Lemma \ref{lem:confiset-typeII}.
 
Invoking  Lemma \ref{lem:confiset-typeII} (a) 
with 
$$
\rho =\frac{\epsilon^2}{H^2\cdot \text{\bedimII}(\Fcal,\Dcal_{\Fcal},\epsilon/H)\cdot|\Acal|}
$$
implies that with probability at least $1-\delta$, for all $(k,h)\in[K]\times[H]$, we have 
	$$\sum_{i=1}^{k-1}  \E \left[ \paren{f_h^k(s_h,\pi_{f_h^k}(s_h)) - (\cT f_{h+1}^k)(s_h,\pi_{f_h^k}(s_h)) }^2\mid s_h \sim \pi^i\right]
		{\le} \mathcal{O} (|\Acal|\beta).$$

Further invoking Lemma \ref{lem:de-regret} with
\begin{equation*}
\left\{
\begin{aligned}
&  \omega=  \frac{\epsilon}{H},\ C=1, \\
&	\Xcal = \Scal, \ \Phi=(I-\Tcal_h)V_\Fcal, \ \Pi=\Dcal_{\Fcal,h},\\
& \phi_k(s):= (f^k_h-\Tcal_h f^k_{h+1})(s,\pi_{f^k_h}(s)) \mbox{ and } \mu_k = \Pr^{\pi^{k}}(s_h=\cdot),
\end{aligned}
\right.
\end{equation*}
we obtain
\begin{equation*}
	\frac{1}{K}\sum_{t=1}^{K} \EcalII(f^{t},\pi^t,h) \le \Ocal\left(\sqrt{\frac{ \text{\bedimII}(\Fcal,\Dcal_{\Fcal},\epsilon/H) |\Acal|\log[KH\Ncal_{\Fcal\cup\Gcal}(\rho)/\delta]}{K}} + 
	\frac{\epsilon}{H}
	\right).
\end{equation*}
Plugging in the choice of $K$ completes the proof.

Similarly, for $\Dcal_\dirac$, we can invoke Lemma \ref{lem:confiset-typeII} (b) witht
$$
\rho =\frac{\epsilon^2}{H^2\cdot \text{\bedimII}(\Fcal,\Dcal_{\dirac},\epsilon/H)\cdot |\Acal|},
$$
and Lemma \ref{lem:de-regret} with
\begin{equation*}
\left\{
\begin{aligned}
&  \omega=  \frac{\epsilon}{H},\ C=1, \\
&	\Xcal = \Scal, \ \Phi=(I-\Tcal_h)V_\Fcal, \ \Pi=\Dcal_{\dirac,h},\\
& \phi_k(s):= (f^k_h-\Tcal_h f^k_{h+1})(s,\pi_{f^k_h}(s)) \mbox{ and } \mu_k = \one\{\cdot=s_h^k\},
\end{aligned}
\right.
\end{equation*}
and obtain 
\begin{equation*}
\begin{aligned}
	\frac{1}{K}\sum_{t=1}^{K} \EcalII(f^{t},\pi^{t},h) 
	\le &\frac{1}{K}\sum_{t=1}^{K} (f^{t}_h-\Tcal f^t_{h+1})(s_h^t,\pi_{f^t_h}(s^t_h)) + \Ocal\paren{\sqrt{\frac{\log K}{K}}} \\
	\le &
	 \Ocal\left(\sqrt{\frac{ \text{\bedimII}(\Fcal,\Dcal_{\dirac},\epsilon/H) |\Acal|\log[KH\Ncal_{\Fcal\cup\Gcal}(\rho)/\delta]}{K}} + 
	\frac{\epsilon}{H}
+\sqrt{\frac{\log K}{K}}	\right),
\end{aligned}
\end{equation*}
where the first inequality follows from standard martingale concentration.

Plugging in the choice of $K$ completes the proof.


\section{Proofs for Examples}
\label{app:proof-examples}

\subsection{Proof of Proposition \ref{prop:eff-eluder}}
\begin{proof}

Suppose $\Fcal$ has finite $\epsilon$-effective dimension and denote the corresponding mapping by $\phi$. 
Then we can rewrite $\Fcal$ in the form  of $\Fcal=\{ f_\theta(\cdot) = \langle \phi(\cdot), \theta \rangle_\cH \mid  \theta\in \Theta \}$, where $\Theta\subset B_\cH(1)$.

Suppose there exists an $\epsilon'$-independent sequence $x_1',\ldots,x_n'\in\Xcal$ with respect to $\Fcal$ where $\epsilon'\ge \epsilon$. 
By the definition of independent sequence, this is equivalent to the existence of   
$\theta_1,\ldots,\theta_n\in (\Theta-\Theta)$ and $x_1,\ldots,x_n\in\phi(\Xcal)$ such that
\begin{equation}
		\left\{
		\begin{aligned}
		&\sum_{i=1}^{t-1} (x_i^\top 	\theta_t)^2 \le \epsilon'^2, \quad t\in[n] \\
		& |x_t^\top 	\theta_t| \ge \epsilon', \quad t\in [n].
		\end{aligned}
\right.
	\end{equation}
Define $\Sigma_t  = \sum_{i=1}^{t-1} x_i x_i^\top + \frac{\epsilon'^2}{4} \cdot \mathrm{I}$. We have 
\begin{equation}
		\begin{aligned}
		\|\theta_t\|_{\Sigma_t} \le \sqrt{2}\epsilon' \quad \Longrightarrow \quad 
		 \epsilon'\le |x_t^\top 	\theta_t| \le \|\theta_t\|_{\Sigma_t}\cdot   \|x_t\|_{\Sigma_t^{-1}}\le \sqrt{2}\epsilon' \|x_t\|_{\Sigma_t^{-1}}, \quad t\in [n].
		\end{aligned}
	\end{equation}
As a result, we should have 	$\|x_t\|_{\Sigma_t^{-1}}^2\ge 1/2$ for all $t\in [n]$.
Now we can apply the standard log-determinant argument,
\begin{align*}
	\sum_{t=1}^{n} \log(1+ \|x_t\|_{\Sigma_t^{-1}}^2) = 
	\log\left(\frac{\det(\Sigma_{n+1})}{\det(\Sigma_1)}\right)
	=  \log\det\left(\mathrm I + \frac{4}{\epsilon'^2} \sum_{i=1}^n x_i x_i^\top \right),
\end{align*}
which implies 
\begin{align}
	0.5 \le \min_{t\in[n]} \|x_t\|_{\Sigma_t^{-1}}^2 
	\le \exp\left(\frac{1}{n} \log\det\left(\mathrm I + \frac{4}{\epsilon'^2} \sum_{i=1}^n x_i x_i^\top \right)\right)-1.
\end{align}
Choose $n=d_{\rm eff}(\Fcal,\epsilon/2)$ that is the minimum positive integer satisfying 
\begin{equation}
	\sup_{x_1,\ldots,x_n \in \phi(\Xcal)}\frac{1}{n} \log\det\left(\mathrm I + \frac{4}{\epsilon^2} \sum_{i=1}^n x_i x_i^\top \right) \le e^{-1}.
\end{equation}
This leads to a contradiction because $\epsilon'\ge\epsilon$ and  $0.5> e^{e^{-1}}-1$. So we must have $$\dim_{\rm E}(\Fcal,\epsilon)\le d_{\rm eff}(\Fcal,\epsilon/2).$$

\end{proof}

\subsection{Proof of Proposition \ref{prop:kernelMDP-cover}}
\begin{proof}
Consider fixed $\epsilon\in\R^+$ and $h\in[H]$, and denote 
$n = \dim_{\rm E}(\Fcal,\epsilon)$.
 Then by the definition of Eluder dimension, there must exist $x_1,\ldots,x_n\in\Xcal_h$ where $\Xcal_h = \{ \phi_h(s,a):~(s,a)\in\Scal\times\Acal\}$ so that 
for any $\theta,\theta'\in B_\cH(H+1-h)$, if 
$\sum_{i=1}^{n}( \langle x_i, \theta-\theta'\rangle_\cH)^2 \le \epsilon^2$, then $|\langle z, \theta-\theta'\rangle_\cH|\le\epsilon$ for any $z\in\Xcal_h$. 
In other words, $x_1,\ldots,x_n$ is one of the longest independent subsequences. 
Therefore, in order to cover $\Fcal_h$, we only need cover the projection of $ B_\cH(H+1-h)$ onto the  linear subspace spanned by $x_1,\ldots,x_n$, which is at most $n$ dimensional. 

By standard $\epsilon$-net argument, there exists $\cC\subset B_\cH(H+1-h)$ such that: (a) $\log|\cC| \le \Ocal(n\cdot\log(1+nH/\epsilon))$, (b) for any $\theta\in B_\cH(H+1-h)$, there exists $\hat\theta\in\cC$ satisfying  
$\sum_{i=1}^{n}( \langle x_i, \theta-\hat\theta\rangle_\cH)^2 \le \epsilon^2$.
By the property of $x_1,\ldots,x_n$,   $\{\phi_h(\cdot,\cdot)\trans\hat\theta~\mid~\hat\theta\in \cC\}$ is an $\epsilon$-cover of $\Fcal_h$.
Since $\Fcal=\Fcal_1\times\cdots\times\Fcal_H$, we obtain 
$\log\Ncal_\Fcal(\epsilon) \le \Ocal\big(Hn \cdot\log(1+nH/\epsilon)\big)$.
Finally, by Proposition \ref{prop:kernelMDP-eluder}, $n\le d(\epsilon)$, which concludes the proof.
	
\end{proof}

\subsection{Proof of Proposition \ref{prop:effective-bellman-bedim}}
\begin{proof}
Assume there exists $h\in[H]$  such that $\dedim((I-\Tcal_h)\Fcal,\Dcal_{\Fcal,h},\epsilon)\ge m$.
	Let $\mu_1,\ldots,\mu_n\in\Dcal_{\Fcal,h}$ be a  an $\epsilon$-independent sequence with respect to $(I-\Tcal_h)\Fcal$.
	 By Definition \ref{def:ind_dist}, there exist $f^1,\ldots,f^n$ such that for all $t\in[n]$, $\sqrt{\sum_{i=1}^{t-1}(\E_{\mu_i}[f^t_h-\Tcal_h f_{h+1}^t])^2}\le \epsilon$ and $|\E_{\mu_t}[f^t_h-\Tcal_h f_{h+1}^t] |>\epsilon$.
	 Since $\mu_1,\ldots,\mu_n\in\Dcal_{\Fcal,h}$, there exist $g^1,\dots,g^n\in\Fcal$ so that $\mu_i$ is generated by executing $\pi_{g^i}$, for all $i\in[n]$.
	 
	By the definition of effective Bellman rank, this is equivalent to: $\sqrt{\sum_{i=1}^{t-1}(\langle \phi_h(f^t),\psi_h(g^i)\rangle)^2}\le \epsilon$ and $| \langle \phi_h(f^t),\psi_h(g^t)\rangle|>\epsilon$ for all $t\in[n]$. For notational simplicity, define $x_i = \psi_h(g^i)$ and $\theta_i = \phi_h(f^i)$. Then
	\begin{equation}
		\left\{
		\begin{aligned}
		&\sum_{i=1}^{t-1} (x_i^\top 	\theta_t)^2 \le \epsilon^2, \quad t\in[n] \\
		& |x_t^\top 	\theta_t| \ge \epsilon, \quad t\in [n].
		\end{aligned}
\right.
	\end{equation}
The remaining arguments follow the same as in the proof of Proposition \ref{prop:eff-eluder} except that we replace $\epsilon$ by $\epsilon/\zeta$.
\end{proof}

\subsection{Proof of Proposition \ref{ex:pomdps}}
\begin{proof}
Note that the case $h=1$ is trivial because each episode always starts from a fixed initial state independent of the policy.
	For any policy $\pi$, function $f\in\Fcal$, and step $h\ge2$ 
\begin{align*}
	\EcalII(f,\pi,h) =& \E[ f_h(o_h,a_h)- r_h(o_h,a_h) -  f_{h+1}(o_{h+1},a_{h+1})\mid s_h\sim\pi, a_{h:h+1}\sim\pi_f]\\
= 	&\E[ f_h(o_h,a_h)- r_h(o_h,a_h) -  f_{h+1}(o_{h+1},a_{h+1})\mid (s_{h-1},a_{h-1})\sim\pi, a_{h:h+1}\sim\pi_f]\\	
=& 	\sum_{s,a\in\Scal} 	\sum_{s'\in\Scal}
\Pr^\pi(s_{h-1}=s,a_{h-1}=a)
\cdot \langle \phi_{h-1}(s,a), \psi_{h-1}(s')\rangle_\cH \cdot \Vcal(s'),
\end{align*}
where 
$$
\Vcal(s') = \E[ f_h(o_h,a_h)- r_h(o_h,a_h) -  f_{h+1}(o_{h+1},a_{h+1})\mid s_h=s', a_{h:h+1}\sim\pi_f].
$$
As a result, we obtain 
\begin{align*}
&\E[ f_h(o_h,a_h)- r_h(o_h,a_h) -  f_{h+1}(o_{h+1},a_{h+1})\mid s_h\sim\pi, a_{h:h+1}\sim\pi_f]\\
=& \bigg\langle \E_\pi[\phi_{h-1}(s_{h-1},a_{h-1})],
\sum_{s'\in\Scal}\psi_{h-1}(s')\Vcal(s')\bigg\rangle_\cH	~.
\end{align*}
Notice that the left hand side of the inner product only depends on $\pi$ while the right hand side only depends on $f$. 
Moreover, by the definition of kernel reactive POMDPs, the RHS has norm at most $2$. 
Therefore, we conclude the proof by revoking Proposition \ref{prop:effective-bellman-bedim-V} with $\zeta=2$.
\end{proof}

\section{Discussions on $\Dcal_{\Fcal}$ versus $\Dcal_{\dirac}$ in BE Dimension}

 In this paper, we have mainly focused on the BE dimension induced by two special distribution families: $(a)$ $\Dcal_{\Fcal}$ --- the roll-in distributions produced by executing the greedy policies induced by the functions in $\Fcal$, $(b)$ $\Dcal_{\dirac}$ --- the collection of all Dirac distributions. And we prove that both low $\text{\bedim}\big(\Fcal,\Dcal_{\Fcal},\epsilon\big)$ and low $\text{\bedim}\big(\Fcal,\Dcal_{\dirac},\epsilon\big)$ can imply sample-efficient learning. 
As a result, it is natural to ask what is the relation between $\text{\bedim}\big(\Fcal,\Dcal_{\Fcal},\epsilon\big)$ and $\text{\bedim}\big(\Fcal,\Dcal_{\dirac},\epsilon\big)$? Is it possible that one of them is always no larger than the other so that we only need to use the smaller one?
We answer this question with the following proposition, showing that either of them can be arbitrarily larger than the other.  

\begin{proposition}\label{lem:FvsDirac}
	There exists absolute constant $c$ such that for any $m\in\N^{+}$, 
	\begin{enumerate}[label=(\alph*)]
		\item  there exist an MDP and a function class $\Fcal$ satisfying 
for all $\epsilon\in(0,1/2]$, 
$\text{\bedim}(\Fcal,\Dcal_{\Fcal},\epsilon)\le c$ while  
$\text{\bedim}(\Fcal,\Dcal_{\dirac},\epsilon)\ge m$.
\item  there exist an MDP and a function class $\Fcal$ satisfying 
for all $\epsilon\in(0,1/2]$, 
$\text{\bedim}(\Fcal,\Dcal_{\dirac},\epsilon)\le c$ while  
$\text{\bedim}(\Fcal,\Dcal_{\Fcal},\epsilon)\ge m$.
	\end{enumerate}
\end{proposition}
\begin{proof}
We prove $(a)$ first.
Consider the following contextual bandits problem ($H=1$).
	\begin{itemize}
		\item 
		There are $m$ states $s_1,\ldots,s_m$ but the agent always starts at $s_1$. This means the agent can never visit other states because each episode  contains only one step ($H=1$).
		\item There are two actions $a_1$ and $a_2$. The reward function is zero for any state-action pair.

		\item The function class $\Fcal_1=\{ f_i(s,a)=\one(s=s_i)+\one(a=a_1) :\  i\in[m]\}$.
	\end{itemize}
First of all, note in this setting $\Dcal_{\dirac}$ is the collection of all Dirac distributions over $\Scal\times\Acal$, 
$\Dcal_{\Fcal,1}$ is a singleton  containing only $\delta_{(s_1,a_1)}$, 
 and $(I-\Tcal_1)\Fcal$ is simply $\Fcal_1$ because $H=1$ and $r\equiv0$. 
 Since $\Dcal_{\Fcal,1}$ has cardinality one, 
it follows directly from definition that $\text{\bedim}(\Fcal,\Dcal_{\dirac},\epsilon)$ is at most $1$. 
Moreover, it is easy to verify that $(s_1,a_2),(s_2,a_2),\ldots,(s_m,a_m)$ is a $1$-independent sequence with respect to $\Fcal$ because we have $f_i(s_j,a_2)=\one(i=j)$ for all $i,j\in[m]$.
As a result, we have $\text{\bedim}(\Fcal,\Dcal_{\dirac},\epsilon)\ge m$ for all $\epsilon\in(0,1]$. 

Now we come to the proof of  $(b)$. 
Consider the following contextual bandits problem ($H=1$).
	\begin{itemize}
		\item 
		There are $2$ states $s_1$ and $s_2$. In each episode, the agent starts at $s_1$ or $s_2$ uniformly at random. 
		\item There are $m$ actions $a_1,\ldots,a_m$. The reward function is zero for any state-action pair.

		\item The function class $\Fcal_1=\{ f_i(s,a)=(2\cdot\one(s=s_1)-1)+0.5 \cdot{\one(a=a_i)} :\  i\in[m]\}$.
	\end{itemize}
First of all, note in this setting $(I-\Tcal_1)\Fcal$ is simply $\Fcal_1$ and the roll-in distribution induced by the greedy policy of $f_i$ is the uniform distribution over $(s_1,a_i)$ and $(s_2,a_i)$, which we denote as $\mu_i$.
It is easy to verify that $\mu_1,\ldots,\mu_m$ is a $0.5$-independent sequence with respect to $\Fcal$ because $\E_{(s,a)\sim\mu_i}[f_j(s,a)]=0.5\cdot\one(i=j)$. Therefore, for all  $\epsilon\in(0,0.5]$,  
$\text{\bedim}(\Fcal,\Dcal_{\Fcal},\epsilon)\ge m$.

Next, we upper bound $\text{\bedim}(\Fcal,\Dcal_{\dirac},\epsilon)$ which is equivalent to $\dedim(\Fcal_1,\Dcal_{\dirac},\epsilon)$ in this problem.
Assume $\dedim(\Fcal_1,\Dcal_{\dirac},\epsilon)=k$.
	Then there exist $g_1,\ldots,g_k\in\Fcal_1$ and $w_1,\ldots,w_k\in\Scal\times\Acal$ such that for all $i\in[k]$, $\sqrt{\sum_{t=1}^{i-1}(g_i(w_i))^{2}}\le \epsilon$ 
	and $| g_i(w_i)|>\epsilon$. 
	Note that we have $|f(s,a)|\in[0.5,1.5]$ for all $(s,a,f)\in\Scal\times\Acal\times\Fcal_1$. 
	Therefore, if $\epsilon>1.5$, then $k=0$; if $\epsilon\le 1.5$, then $k\le 10$ because $0.5\times\sqrt{k-1}\le \sqrt{\sum_{t=1}^{k-1}(g_k(w_t))^{2}}\le \epsilon\le 1.5$. 
\end{proof}


\end{document}